\documentclass{article}

\usepackage[nonatbib,preprint]{neurips_2024}

\usepackage[utf8]{inputenc} 
\usepackage[T1]{fontenc}    
\usepackage{hyperref}       
\usepackage{url}            
\usepackage{booktabs}       
\usepackage{amsfonts}       
\usepackage{nicefrac}       
\usepackage{microtype}      
\usepackage{xcolor}         
\usepackage{graphicx}
\usepackage{biblatex}
\usepackage{amsmath}
\usepackage{algpseudocode}  
\usepackage{algorithm}
\usepackage{xcolor}
\usepackage{amsthm}
\usepackage[normalem]{ulem}
\usepackage{placeins}

\definecolor{goodbluebar}{RGB}{114, 147, 203}
\definecolor{goodorangebar}{RGB}{225, 151, 76}
\definecolor{goodgreenbar}{RGB}{132, 186, 91}
\definecolor{goodredbar}{RGB}{211, 94, 96}
\definecolor{goodblackbar}{RGB}{128, 133, 133}
\definecolor{goodpurplebar}{RGB}{144, 103, 167}
\definecolor{goodwinebar}{RGB}{171, 104, 87}
\definecolor{goodgoldbar}{RGB}{204, 194, 16}

\definecolor{goodblue}{RGB}{57, 106, 177}
\definecolor{goodorange}{RGB}{218, 124, 48}
\definecolor{goodgreen}{RGB}{62, 150, 81}
\definecolor{goodred}{RGB}{204, 37, 41}
\definecolor{goodblack}{RGB}{83, 81, 84}
\definecolor{goodpurple}{RGB}{107, 76, 154}
\definecolor{goodwine}{RGB}{146, 36, 40}
\definecolor{goodgold}{RGB}{148, 139, 61}

\definecolor{keywordcolor}{RGB}{0,0,255}
\definecolor{commentcolor}{RGB}{0,128,0}
\definecolor{stringcolor}{RGB}{255,0,0}

\bibliography{references}

\title{10 Years of Fair Representations: Challenges and Opportunities}

\author{
  Mattia Cerrato~\thanks{These authors contributed equally.}~~\thanks{Institute of Computer Science Johannes Gutenberg-Universität Mainz, Germany} \\
  \texttt{mcerrato[@]uni-mainz.de} \\
  \And
  Marius Köppel~\footnotemark[1]~~\thanks{Institute for Particle Physics and Astrophysics, ETH Zurich, Switzerland} \\
  \texttt{mkoepp[@]phys.ethz.ch} \\
  \And
  Philipp Wolf~\footnotemark[1]~~\footnotemark[2] \\
  \texttt{pwolf01[@]students.uni-mainz.de} \\
  \And
  Stefan Kramer~\footnotemark[2] \\
  \texttt{kramer[@]informatik.uni-mainz.de} \\
}

\begin{document}

\maketitle

\begin{abstract}
Fair Representation Learning (FRL) is a broad set of techniques, mostly based on neural networks, that seeks to learn new representations of data in which sensitive or undesired information has been removed.
Methodologically, FRL was pioneered by Richard Zemel et al. about ten years ago. The basic concepts, objectives and evaluation strategies for FRL methodologies remain unchanged to this day. 
In this paper, we look back at the first ten years of FRL by i) revisiting its theoretical standing in light of recent work in deep learning theory that shows the hardness of removing information in neural network representations and ii) presenting the results of a massive experimentation (225.000 model fits and 110.000 AutoML fits) we conducted with the objective of improving on the common evaluation scenario for FRL.
More specifically, we use automated machine learning (AutoML) to adversarially "mine" sensitive information from supposedly fair representations.
Our theoretical and experimental analysis suggests that deterministic, unquantized FRL methodologies have serious issues in removing sensitive information, which is especially troubling as they might seem "fair" at first glance.
\end{abstract}

\section{Introduction}

Biased machine learning systems have been shown to have detrimental impacts on society, perpetuating social inequalities and reinforcing harmful stereotypes.
For instance, in Amazon's attempt to automate its hiring process, the company's computer programs, developed since 2014, reportedly aimed to streamline talent acquisition by analyzing resumes.
However, the system was reported to display gender bias, penalizing resumes containing terms like ``women's,'' disadvantaging female applicants for technical roles~\cite{amazon}.
Similarly, in the United States, algorithms like COMPAS have been used in nine states to assess a criminal defendant's risk of recidivism.
An analysis of COMPAS revealed discriminatory outcomes: black defendants who did not recidivate were more frequently misclassified as high risk compared to their white counterparts, while white re-offenders were often mislabeled as low risk~\cite{machine_bias}.

Both examples show the concern that such models trained on biased data might then learn those biases~\cite{barocas-hardt-narayanan}, therefore perpetuating historical discrimination against certain groups of individuals.
Machine learning methodologies designed to avoid these situations are often said to be ``group-fair'' in the sense that they seek to distribute resources equally across groups.
This paper focuses on a specific kind of algorithm -- Fair Representation Learning (FRL) -- which is part of this domain.

FRL is a broad set of techniques that seek to remove undesired information from data.
FRL is based mostly but not exclusively on neural network techniques.
Our focus in this paper is however the theoretical and experimental evaluation of neural network-based FRL.
The goal of such techniques is to learn the parameters $\theta$ for a projection $f_\theta: X \rightarrow Z$ from the feature space $X$ to a latent feature space $Z$.
The task was pioneered by Zemel et al., about ten years ago~\cite{zemel2013learning}.

The two competing goals for $Z$ are to remove all information about a sensitive attribute $S$ while retaining as much information as possible about some task for which labeled data $Y$ is available.
An alternative formulation is based on the auto-encoding concept: information about $X$ should still be present as much as possible in $Z$.
While it is of course possible to simply remove $S$ from the dataset columns, this does not generally prevent statistical inference on $S$.
Discarding sensitive data is usually termed ``fairness by unawareness'' and does not in general grant group-anonymity (we refer the interested reader to the book by Barocas et al. \cite{barocas-hardt-narayanan}, Chapter 3, Figure 3.4). 
A simple way to understand this phenomenon is to reason about the correlation between ZIP code (sometimes deemed non-sensitive information) and ethnicity (usually deemed sensitive) in the US and other countries.
FRL improves on fairness by unawareness by actively seeking to ``stamp out'' and remove any correlation between the learned representation $Z$ and the sensitive information $S$.
It has been observed in practice in the last 10 years that information removal contributes to other fairness metrics such as independence/disparate impact or separation/disparate mistreatment \cite{xie2017controllable,madras,cerrato2020constraining,Louizos2016TheVF,he2016deep}.

One advantage of these methodologies is that any classifier can be trained on the learned fair representation~\cite{zemel2013learning}, while other methods may rely on a specific model or technique.
Due to this, FRL enables a ``separation of concerns'' scenario~\cite{mcnamara2017provably}. 
Here, a data user is assumed to be interested in developing an ML-based automated or semi-automated decision-making system for which fairness concerns are relevant. A trusted data regulator, who is also allowed access to sensitive information, will then employ an FRL algorithm and share the obtained fair representations privately with the data user.
This setup gives the opportunity for increased trust into the overall ML-based decision-making system, as the regulator would able to evaluate the amount of correlation between $Z$ and $S$ while the user will not have access to $S$ or any of its correlations.
It is relatively common for work in FRL to perform the above investigation by training some number of classifiers on $Z$ and observing whether their performance is close enough to random guessing for the dataset at hand.
If it is, then this provides some empirical evidence that an FRL method is working as intended.

All these advantages notwithstanding, it is not straightforward to conclude that FRL should be the go-to methodology for fairness-sensitive applications.
One significant limitation here is that neural network-based FRL is not transparent and quite hard to interpret \cite{interFair}.
When fairness is relevant, the application is by default high-stakes \cite{rudin2019stop}: it is then hard to justify employing neural networks, esp. when the data is tabular and other methodologies are therefore better than, or at least competitive with, FRL \cite{grinsztajn2022tree}.
The one advantage that remains unique to FRL is therefore, in our view, the aforementioned separation of concerns scenario.

In this paper, we revisit the first 10 years of fair representation learning and discuss its unique limitations and opportunities for real-world impact. We start by summarizing the most visible contributions in this area and how they relate to one another. Then, we discuss the general theoretical setup of FRL and discuss its limitations by relating them to theoretical advancements in understanding the information dynamics of deep neural networks \cite{goldfeld2020information}. We then move on to showing the result of a massive experimentation -- a total of around 225.000 model fits and 110.000 AutoML fits -- we ran across 6 datasets. We release \texttt{EvalFRL}, the experimental library we developed for severe testing of FRL methodologies, which can found at \url{https://anonymous.4open.science/r/EvalFRL/}.

\section{Related Work}
Algorithmic fairness has garnered considerable interest from both academia and the general public in recent years, largely due to the ProPublica/COMPAS controversy \cite{machine_bias,Rudin2020Age}. However, the earliest contribution in this field appears to date back to 1996, when Friedman and Nisselbaum \cite{friedman1996bias} highlighted the necessity for automatic decision systems to be aware of systemic discrimination and ethical considerations. The importance of addressing automatic discrimination is also reflected in EU legislation, particularly in the GDPR, Recital 71 \cite{malgieri2020gdpr}.
One approach to addressing these issues involves eliminating the influence of the "nuisance factor" $S$ from the data $X$ through fair representation learning. This method involves learning a projection of $X$ into a latent feature space $Z$ where all information about $S$ has been removed. 
A pioneering contribution to this area is by Zemel et al. \cite{zemel2013learning}. Since then, neural networks have been widely employed in this context. Some approaches \cite{xie2017controllable, madras} use adversarial learning, a technique introduced by Ganin et al. \cite{ganin2016jmlr}, which involves two networks working against each other to predict $Y$ while removing information about $S$. 
Another line of research \cite{Louizos2016TheVF, moyer2018invariant} uses variational inference to approximate the intractable distribution $p(Z \mid X)$. 
This involves combining architectural design \cite{Louizos2016TheVF} and information-theoretic loss functions \cite{moyer2018invariant, gretton2012kernel} to promote the invariance of neural representations with respect to $S$.
Recently, neural architectures have been proposed for other fairness-related tasks such as fair ranking \cite{zehlike2018reducing,cerrato2020pairwise,fair_pair_metric} and fair recourse \cite{shubham2021fair}.

Another line of investigation that focuses on the information theory of DNNs and provides context for this work is the information bottleneck (IB) problem \cite{tishby2000information} and its applications to the understanding of deep neural networks (DNNs) training dynamics. 
Originally, Swartz-Ziv and Tishby \cite{shwartzziv2017opening} put forward the idea of computing the mutual information term $I(X;Z)$ via quantization and observed that deeper networks undergo a faster compression phase -- a reduction of $I(X;Z)$ that happened earlier in the training process.
These results inspired a reproducibility study by Saxe et al. \cite{saxe2019information},  who observed information compression in networks that employ certain non-linearities (tanh, sigmoids), but no compression when other activations were considered (ReLU).
With regard to the original investigation \cite{shwartzziv2017opening}, Polyanskiy and Goldfeld \cite{goldfeld2020information} retorted that computing $I(X;Z)$ via quantization introduces quantization artifacts and that compression of $I(X;Z)$ is theoretically impossible in deterministic DNNs with injective or bi-Lipschitz activation functions. 
Another limitation was described by Amjad and Geiger \cite{amjad2019learning}, who contributed an analysis of the IB framework under discrete datasets, concluding that the IB functional (its optimization objective) is piecewise constant and is therefore hard to optimize with gradient descent and its variants.
To the best of our knowledge, these fundamental results in the information theory of deep learning have not been analyzed in the context of FRL.

\section{Challenges in Fair Representation Learning}\label{sec:challenge}
Let us denote the dataset of individuals as a matrix $X \in \mathcal{X}^{n \times d}$, where each individual $i \in {1 \dots n}$ is described by a feature vector $x_i$ with $d$ dimensions.
The sensitive attribute is denoted with the random variable $S \in \mathcal{S}$, and the corresponding labels are denoted as $Y \in \mathcal{Y}$.
In fair representation learning, the goal is to learn a representation $Z \in \mathcal{Z}^{n \times m}$ of the data such that it preserves relevant information for the task at hand while removing information about $S$. Usually, but not necessarily, $m < d$. With a slight abuse of notation, we will discuss $X, Y, Z$ as random variables with their sample spaces being $\mathcal{X}, \mathcal{Y}, \mathcal{Z}$, respectively.
Concretely, we define $\phi^i(x) = \sigma(A^i \phi^{i-1}(x) + b^i)$ as the $i$-th layer function of a deep neural network with $L$ layers, where $A$ is a matrix of real-valued weights and $b$ is a bias vector. We assume that $\sigma$ is applied to each dimension of its argument without any aggregation, as is common in DNNs. We note that $\phi^0(x) = x$ and $\phi^L(x) = \hat{y} \in \hat{Y}$, the prediction or reproduction of $Y$. Thus, we name $Z^i = \phi^i(x)$ as the random variable representing the representation extracted from the data by the $i$-th layer of the network.

It follows that if for some $i < L$ it is true that $Z^i \bot S$, then the output $\hat{Y}$ of the network will also be independent of $S$ \cite{madras}, leading for instance to the ``independence'' definition of group fairness in classification \cite{barocas-hardt-narayanan}.
To achieve this goal, most fair representation learning approaches employ a loss function with two terms: a classification loss to ensure predictive performance on the task of interest, and a fairness loss to encourage fairness in the learned representation.
Therefore, the overall objective function for fair representation learning can be formulated as a combination of the classification loss and the fairness loss.
This can be achieved using a weighted sum of the two losses, where the relative importance of each component is controlled by the hyperparameter $\gamma$:

\begin{equation}
    \min_{\theta} \; (1 - 
    \gamma) \mathcal{L}_{\text{class}}(\theta) + \gamma \mathcal{L}_{\text{fair}}(\theta) \label{eq:tradeoff}
\end{equation}

where $\theta$ represents the parameters of the model, $\mathcal{L}_{\text{class}}$ is the classification loss, $\mathcal{L}_{\text{fair}}$ is the fairness loss.
%
As discussed by various authors \cite{achille2018emergence,cerrato2023invariant}, it is possible to formulate this task in an information-theoretic manner by relying on mutual information.
A theoretical formulation of fair representation learning using mutual information between $Z$ and $S$ can be defined by starting from the mutual information between representation and sensitive data:
\begin{equation}
    \mathcal{L}_{fair}(\theta) = I(Z; S) = \int_{s \in S} \int_{z \in Z} P(z, s) \log \frac{P(z, s)}{P(z)P(s)}
\end{equation}
where $P(z, s)$ is the joint probability density and $P(z)$, $P(s)$ are the marginal probability distributions of $Z$ and $S$, respectively. Usually, in FRL $S$ is taken to be discrete, representing some quantized sensitive characteristic that may lead to unacceptable harm, discrimination, or both. 
To achieve a fair or invariant representation, this term needs to be minimized.
Ideally, at the same time the representation $Z$ would be informative for the prediction task, i.e., it would retain sufficient information about the labels $Y$.
\begin{align}
    \min_{Z} \; &I(Z; S).  \label{eq:ib} \\ 
    s.t. \; &I(Z; Y) \geq \alpha \nonumber
\end{align}
The trade-off between preserving task-relevant information and minimizing the mutual information with the sensitive attribute is the key challenge in fair representation learning.

FRL techniques are commonly evaluated over two different frames:
\begin{itemize}
    \item \textbf{Fair allocation.} Suppose that certain values of $\hat{Y}$ lead to desirable outcomes for the individuals represented in $X$. In classification, this may be easily understood as $\hat{Y} = 1$ representing, for instance, being selected for a job interview by a CV-scanning application. Then, a FRL technique succeeds if obtains a fair allocation of desirable outcomes by removing information about $S$ in $Z$ and then using $Z$ in a further classification stage of the network. The fairness of the allocation is then computed via any application-relevant metric, e.g. discrimination \cite{zemel2013learning,xie2017controllable,Louizos2016TheVF,cerrato2020constraining}, disparate mistreatment \cite{zafar2017fairness}, etc..
    \item \textbf{Invariant representations.} Suppose that the representation $Z$ is computed by some trusted party that is allowed access to both $X$ and $S$. Then, a FRL technique succeeds if it may be employed by this trusted party to obtain $Z $ such that $Z \bot S$. In practice, this may be evaluated by training a supervised classifier on $Z$ and computing its accuracy in predicting $S$. Invariant representations may then be safely distributed to data users which may use them to train any ML methodology which will be by construction unaware about $S$ and any of its correlations to $X$.
\end{itemize}

Fair allocation is a high-stakes task for which, however, interpretability may be required, on ethical~\cite{rudin2019stop} or even legal~\cite{malgieri2020gdpr} grounds.
Interpretable FRL is an active area of research \cite{jovanovic2023fare,corrvector} and it is in general not straightforward to interpret the meaning of $Z$.
Currently, it may be preferable to employ better-understood methodologies, such as fair reductions~\cite{zafar2019constraints}. If it is acceptable to use $S$ at test time, post-processing techniques are provably optimal~\cite{hardt2016equality}.
Thus, learning invariant representations would be the main -- and nominal -- selling point of FRL.
We report in the following some relatively well-known results in the information theory of deep learning whose consequences for FRL, to the best of our knowledge, have not been previously discussed.
\subsection*{Information and Mutual Information in Neural Networks}
The optimization problem in Eq. \ref{eq:ib} bears a close resemblance to the information bottleneck (IB) problem introduced in \cite{tishby2000information} and then famously applied to neural network training dynamics \cite{tishby2015deep}. 
The authors propose to understand learning representations as the problem of compressing $X$ into $Z$ while losing minimal information about $Y$. 
The only significant difference with the information-theoretic formulation of FRL is then that $X$ is substituted by $S$.
Then, we prove in the following that previous work on mutual information in deep neural networks \cite{goldfeld2020information,czyz2024beyond} also applies to FRL:
\newtheorem{mi}{Theorem}
\begin{mi}\label{th:theorem1}
Let X, Y, and S be the random variables representing data, labels and the sensitive attributes, respectively. Let $\phi^i(x) = \sigma(A^i \phi^{i-1}(x) + b^i)$ be the $i$-th layer function of a DNN, where $A$ is a weight matrix, $b$ a bias vector, and $\sigma$ an injective non-linearity. Let $Z^i = \phi^i(x)$ be a random variable. Let thus $S \to Y \to X \to \dots \to Z^k \to \dots \hat{Y}$ be a Markov chain, where $\hat{Y}$ is an estimation of $Y$. Then, $I(X;S) = I(Z^i;S) \; \forall i \in \{1 \dots L\}$,  where $L$ is the number of layers in the network. 
\end{mi}
\begin{proof}
We note that each $Z^i = \phi^i(x)$ is a deterministic, one-to-one mapping of the previous layer's input, or of the input itself. Thus, $H(X) = H(Z^i) \; \forall i \in \{1 \dots M\}$ \cite{goldfeld2020information,czyz2024beyond}.
Then, we rewrite the mutual information terms as follows:
\begin{align*}
I(S;X) &= H(X) - H(X \mid S) \\
I(S;Z^i) &= H(Z^i) - H(Z^i \mid S)
\end{align*}
since $H(X) = H(Z^i)$, it follows that the theorem is true if 
\begin{align}
H(X \mid S) = H(Z^i \mid S) \label{eq:entropies}
\end{align}
We first note that the joint entropies $H(Z^i, S)$ and $H(X, S)$ are equal as $Z^i$ is computed via a one-to-one mapping of $X$ \cite{PolyanskiyWu2023}. Then, by the chain rule of entropy, we have $H(Z^i|S) = H(Z^i,S) - H(S)$ and $H(X|S) = H(X,S) - H(S)$. Substituting these equalities in Eq.~\ref{eq:entropies} concludes the proof.
\end{proof}
This result derives straightforwardly from the fact that deterministic neural networks with injective activation functions are one-to-one mappings of the input data.
Thus, two different $x_1, x_2 \in \mathcal{X}^{n \times d}$ will always be mapped onto two different representations $z^i_1, z^i_2 \in \mathcal{Z}_i^{n \times m}$, even if $m < d$.
It follows that sensitive information is not removed in general when descending the layers of a neural network.
Therefore, Theorem \ref{th:theorem1} is an impossibility theorem for FRL on infinite-precision deterministic networks with tanh or sigmoid activations, and may be extended to bi-Lipschitz functions \cite{goldfeld2020information} such as Leaky-ReLU.
It is important to note that non-injective activations such as ReLU escape the theorem; as pointed out by Amjad and Geiger \cite{amjad2019learning}, however, another practical limitation applies.
Specifically, $I(Z;X)$ will only take finitely many values provided that the data features $X$ is discrete.
In FRL, $S$ is usually assumed to be discrete and thus $I(Z;S)$ is piecewise constant, which makes for a difficult objective to optimize for. 
Another caveat is provided by the fact that invariance in the mutual information does not imply invariance in its estimation.
Supervised classifiers trained on $(Z, S)$ may as well display a lesser degree of accuracy compared to ones trained on $(X, S)$ -- as data samples with different values of $s \in \mathcal{S}$ get mapped closer together, it may be harder in practice to distinguish between them.
It is also critical to note here as well that $Z^i$ does not have infinite precision on a discrete computer~\cite{saxe2019information}.
The sigmoid function only saturates to 1 as $x \to \infty$; however, a computer will return $1$ as the result of $\frac{1}{e^{-x}+1}$ much sooner than that, depending on how many bits are used to represent the result of the computation.
Using low-bit representations, or ``hard'' clusterings, is therefore another way to escape the limitations put forward by Theorem \ref{th:theorem1} \cite{cerrato2023invariant,jovanovic2023fare}.
Lastly, we note that Theorem \ref{th:theorem1} does not apply to neural network models that incorporate stochasticity in their process, as $H(X)$ is in general not equal to $H(Z^i)$ in that situation.

\section{Experiments}

In this section, we report on an in-depth experimentation that we conducted with the aim of evaluating the significance of the impossibility theorem stated in the previous section.
We are especially interested in testing whether it is still possible to predict the sensitive attribute $S$ with a high degree of accuracy from the fair representations learned by deterministic FRL methodologies, as our theoretical result suggests that it should be so.
To this end, we evaluated a total of 8 FRL methodologies: BinaryMI and DetBinaryMI, \cite{cerrato2023invariant}, DebiasClassifier \cite{ganin2016jmlr,xie2017controllable,madras}, NVP \cite{fairNVP}, VFAE \cite{Louizos2016TheVF}, ICVAE \cite{moyer2018invariant}, LFR \cite{zemel2013learning} and Deep Domain Confusion \cite{tzeng2014deep}.
While some of these are fully determinisic, others incorporate stochasticity by sampling from a parametric distribution which parameters are learned in the network.
Our expectation is that stochastic models will show greater invariance in their learned representations, as Theorem \ref{th:theorem1} does not apply and sensitive information may in principle be removed.
We tested these models on six different commonly employed fairness-sensitive datasets, focusing on the task of fair classification and independence as a fairness definition (i.e. $\hat{Y} \bot S$). 
We give in-depth descriptions of models and datasets in the Appendix, Sections~\ref{sec:app-models} and \ref{sec:app-datasets}, respectively.

We release \texttt{EvalFRL}\footnote{\url{https://anonymous.4open.science/r/EvalFRL/}}, the experimental framework we employed to perform our experimentation alongside all experimental metadata, including best hyperparameters for all models and performance at the outer fold level\footnote{\url{https://drive.google.com/drive/folders/1koZd8cgBJMVGuH3uRqvpTFEUJo0Sd23q?usp=sharing}} \footnote{We discuss the used compute infrastructure to run the experiments in the Appendix \ref{sec:compute}.}. Our framework was developed to perform in-depth evaluation of FRL methodologies across the two main frames discussed in Section \ref{sec:challenge} -- fair allocation and invariant representations. Thus, we evaluate both the fairness of the allocations learned by FRL methods and approximate the mutual information between the representations $Z$ and the sensitive attribute $S$ with the performance of external estimators.
%
%
\subsection{\texttt{EvalFRL}: An Evaluation Library for Fair Representation Learning}
As mentioned above, if an estimator can predict $S$ better than random guessing, this indicates that information on $S$ is still contained in the fair representation.
To investigate whether this may happen in commonly employed FRL techniques, we developed the \texttt{EvalFRL} framework, wherein every tested dataset-model combination follows a standardized testing pipeline.
This process is fully reproducible, thereby ensuring comparability between the models.
%
In previous work tackling FRL, the experimental setup has focused on training a few classifiers on $(Z^i, S)$ \cite{cerrato2020constraining,zemel2013learning,Louizos2016TheVF,xie2017controllable}. 
However, the number of classifiers may vary and the optimal hyperparameters are not always reported, leading to results which are hard to compare across different papers.
We employ automated machine learning (AutoML) to handle this problem.
AutoML automatically searches for the optimal machine learning solution for a given problem.
This includes, among others, feature preprocessing, model selection and hyperparameter tuning. 

\begin{algorithm}
    \caption{Overview of \texttt{EvalFRL} logic when ran for one dataset-model combination.}
    \label{alg:framework}
    \begin{algorithmic}
        \State $\textcolor{goodblue}{r} \gets 15$, $\textcolor{goodblue}{k} \gets 3$, $\textcolor{goodblue}{seed} \gets 123$, $\textcolor{goodblue}{results} \gets []$, $\textcolor{goodblue}{repr_{\text{train}}} \gets []$, $\textcolor{goodblue}{repr_{\text{test}}} \gets []$
        \For{$i$ \textcolor{goodblue}{in} $1$ \textcolor{goodblue}{to} $r$}
            \State $X_{\text{train}}, X_{\text{test}}, y_{\text{train}}, y_{\text{test}}, s_{\text{train}}, s_{\text{test}} \gets \text{\textsc{train\_test\_split}}(X, y, s, \frac{2}{3}, \text{random\_state} = seed)$
            \State $cv \gets \text{\textsc{RandomizedSearchCV}}(model, param\_distributions, cv=k, n\_iter=100)$
            \State $\text{\textsc{bestmodel}} \gets cv.\text{\textsc{fit}}(X_{\text{train}}, y_{\text{train}}, s_{\text{train}}, \gamma=0)$
            \For{$\gamma_0$ \textcolor{goodblue}{in} $\{0, 0.1, \dots 1\}$}
                \State $\text{\textsc{bestmodel.fit}}(X_{\text{train}}, y_{\text{train}}, s_{\text{train}}, \gamma=\gamma_0)$ \Comment{\textcolor{goodgreen}{bestmodel is fit from scratch}}
                \State $results.\text{\textsc{append}}(\text{\textsc{evaluation}}(\text{\textsc{bestmodel}}, X_{\text{test}}, y_{\text{test}}, s_{\text{test}}))$
                \State $repr_{\text{train}}.\text{\textsc{append}}(\text{\textsc{get\_representations}}(\text{\textsc{bestmodel}}, X_{\text{train}}))$
                \State $repr_{\text{test}}.\text{\textsc{append}}(\text{\textsc{get\_representations}}(\text{\textsc{bestmodel}}, X_{\text{test}}))$
            \EndFor
            \For{$repr_{train_i}$, $repr_{test_i}$ \textcolor{goodblue}{in} ($repr_{\text{train}}, repr_{\text{test}}$)}
                \State $AutoML \gets \text{\textsc{automl}}()$ \Comment{\textcolor{goodgreen}{AutoML is trained from scratch at each CV iteration}}
                \State $AutoML.\text{\textsc{fit}}(repr_{train_i}, s_{\text{train}})$
                \State $results.\text{\textsc{append}}(\text{\textsc{evaluation}}(AutoML, repr_{test_i}, s_{\text{test}}))$
            \EndFor
        \EndFor
    \end{algorithmic}
\end{algorithm}

We show in Algorithm \ref{alg:framework} the main use case of \texttt{EvalFRL}.
A detailed description of the steps and a graphical representation of the framework's logic is shown in the Appendix \ref{app:exp}.
The whole pipeline is built using the Kedro framework \cite{Alam_Kedro_2024} and can be easily extended to include other models, datasets and metrics beyond the ones we consider in this work.
The data preprocessing step starts with the segmentation of the data into features $X$, the label $y$ and the sensitive attribute $S$.
Additionally, $y$ gets transformed to either a positive ($y=1$) or a negative ($y=0$) outcome and $S$ to either the privileged ($S=1$) or underprivileged ($S=0$) group.
Categorical features undergo encoding, while continuous features are normalized with mean 0 and variance 1.
In the subsequent step, hyperparameter-tuning is performed utilizing $r$-times $k$-fold cross-validation \cite{bouckaertfrank2004}.
In our experiments we set $r=15$ and $k=3$ by following the recommendations in Naudeau and Bengio \cite{nadeau2003inference}.
The best hyperparameters found in every outer-loop, along with a range of $\gamma$ values regulating the fairness/accuracy tradeoff (Equation \ref{eq:tradeoff}), are then utilized to evaluate the model.
The evaluation step contains the model's predictive performance on the label $y$ (Acc, AUC), as well as multiple fair allocation metrics.
Finally, we seek to evaluate $Z$ and $S$, for each fair representation generated in the previous step.
We utilize an AutoML library called MLJAR \cite{mljar} for this process.
AutoML searches the best model and trains it using the representations on the train-data and the corresponding sensitive information $S$.
Its performance gets tested using the representations on the held out test data and $S$.

The output of the framework allows for an exploration of the trade-off between predicting the label $y$ and the fairness of the model according to known fairness-metrics.
Additionally, it provides researchers with a common, high-effort evaluation framework for FRL methodologies in the invariant representation framing. 
If it is possible to have a higher performance than guessing $S$ randomly, this leads to the conclusion that there is still information on $S$ contained in the representation.
These investigations are done over a range of $\gamma$ tradeoff values, which makes it possible to understand the impact of the $\gamma$ parameter on both representation invariance and fairness of allocations.
\subsection{Fairness Metrics}

\paragraph*{Area under Discrimination Curve (AUDC)}
We quantify disparate impact through discrimination, following the approach introduced by Zemel et al. \cite{zemel2013learning}.
The discrimination metric, denoted as $\text{yDiscrim}$, is defined as:

\begin{equation}
    \text{yDiscrim} = \left | \frac{\sum^n_{n:s_n=1} \hat{y}_n}{\sum^n_{n:s_n=1} 1} - \frac{\sum^n_{n:s_n=0} \hat{y}_n}{\sum^n_{n:s_n=0} 1} \right |,
\end{equation}
where $n:s_n=1$ indicates that the $n$-th example has a value of $s$ equal to 1.
To generalize this metric akin to how accuracy generalizes to obtain a classifier's area under the curve (AUC), we evaluate the above measure for different classification thresholds and then compute the area under this curve.
In our experiments, we utilized $100$ equispaced thresholds.
We call this measure AUDC, following conventions established in the literature \cite{interFair}.
Unlike AUC, lower values are indicative of better performance.
\paragraph*{rND}
To measure fairness in learning to rank applications, we use the rND metric~\cite{ke2017measuring}.
This metric evaluates differences in exposure across multiple groups and is defined as:
\begin{equation}
\text{rND} = \frac{1}{Z} \sum_{i \in \{10, 20, ...\}}^N \frac{1}{\log_{2}(i)} \left| \frac{ \mid S^{+}_{1...i} \mid}{i} - \frac{\mid S^+ \mid}{N} \right |.
\end{equation}
rND measures the difference between the ratio of the protected group in the top-$i$ documents and in the overall population.
The maximum value, $Z$, serves as a normalization factor and is computed with a dummy list where all members of the underprivileged group are placed at the end, representing ``maximal discrimination.''
The metric also penalizes over-representation of protected individuals at the top compared to their overall population ratio.
\subsection{Results in Fair Allocation}\label{sec:results-fair-alloc}
\begin{figure}[h]
    \centering
    \includegraphics[width=\linewidth]{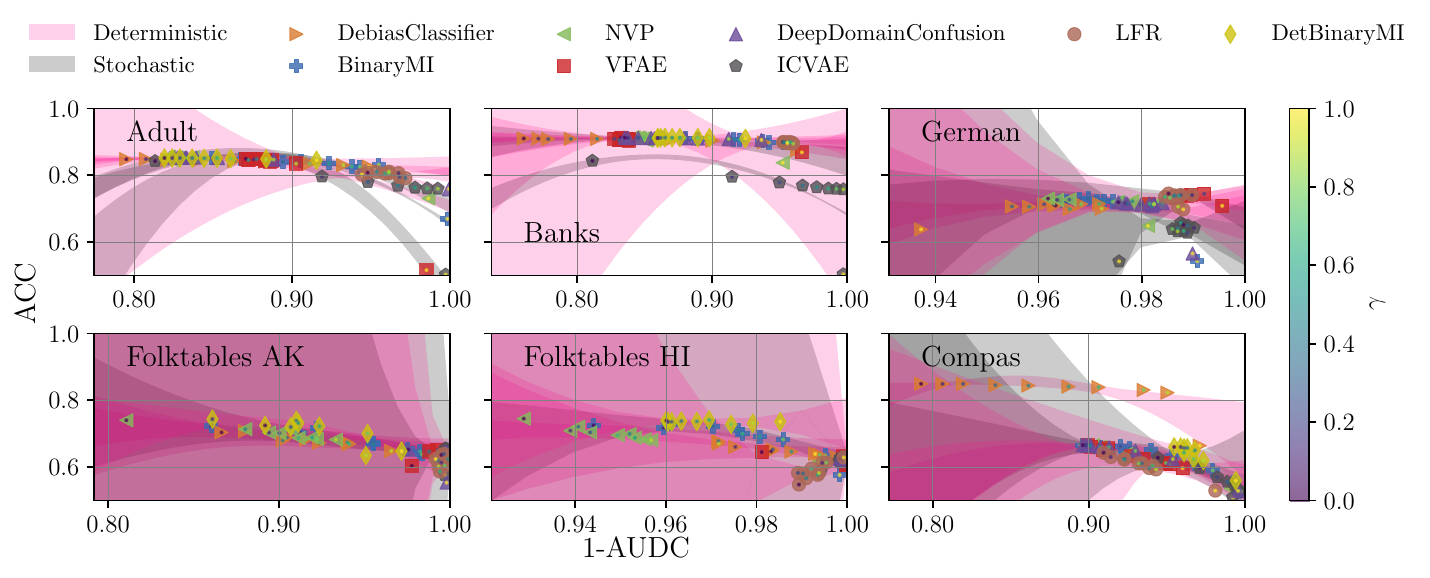}
    \caption{Accuracy vs. 1 - AUC-Discrimination tradeoff for all six dataset and eight model combinations. Each model is displayed for different $\gamma$ values indicated via a colored point inside the model marker.}
    \label{fig:model-acc-auc-dis}
\end{figure}
The model accuracy vs. fairness tradeoff results are shown in Figure \ref{fig:model-acc-auc-dis} (ACC vs. 1-AUDC) and \ref{fig:model-auc-auc-dis} (AUC vs. 1-AUDC), as well as in Figure \ref{fig:model-acc-dis} (ACC vs. 1-Discrimination), \ref{fig:model-auc-dis} (AUC vs. 1-Discrimination), \ref{fig:model-acc-parity} (ACC vs. Statistical Parity Difference), \ref{fig:model-auc-parity} (AUC vs. Statistical Parity Difference), \ref{fig:model-acc-delta} (ACC vs. Delta), \ref{fig:model-auc-delta} (AUC vs. Delta), \ref{fig:model-acc-rnd} (ACC vs. 1-rND), \ref{fig:model-auc-rnd} (AUC vs. 1-rND) of the appendix. 
Each figure shows all combinations of the six datasets and eight models.
The colored points in the symbols of the models show the used $\gamma$ value, while the colored areas show the variance employing 100 Gaussian bootstrapping fits using the mean and variance of the model performance optained from the 15 hold-out splits.
For most datasets and metric combinations (e.g. Adult, Banks and German) one can observe that the performance of the the majority of the models is equal.
Most models show a well defined tradeoff behaviour when changing the $\gamma$ value.

In Figure \ref{fig:model-acc-auc-dis} we observe that the performance for the DebiasClassifier on the Compas dataset outperforms all other models.
However, it takes mostly fair allocation decisions at higher values of gamma.
We conclude that the performance of FRL methodologies in the task of fair allocation is approximately equal. Varying the tradeoff parameter $\gamma$, as expected, leads to fairer decisions.

\begin{figure}[h]
    \centering
    \includegraphics[width=\linewidth]{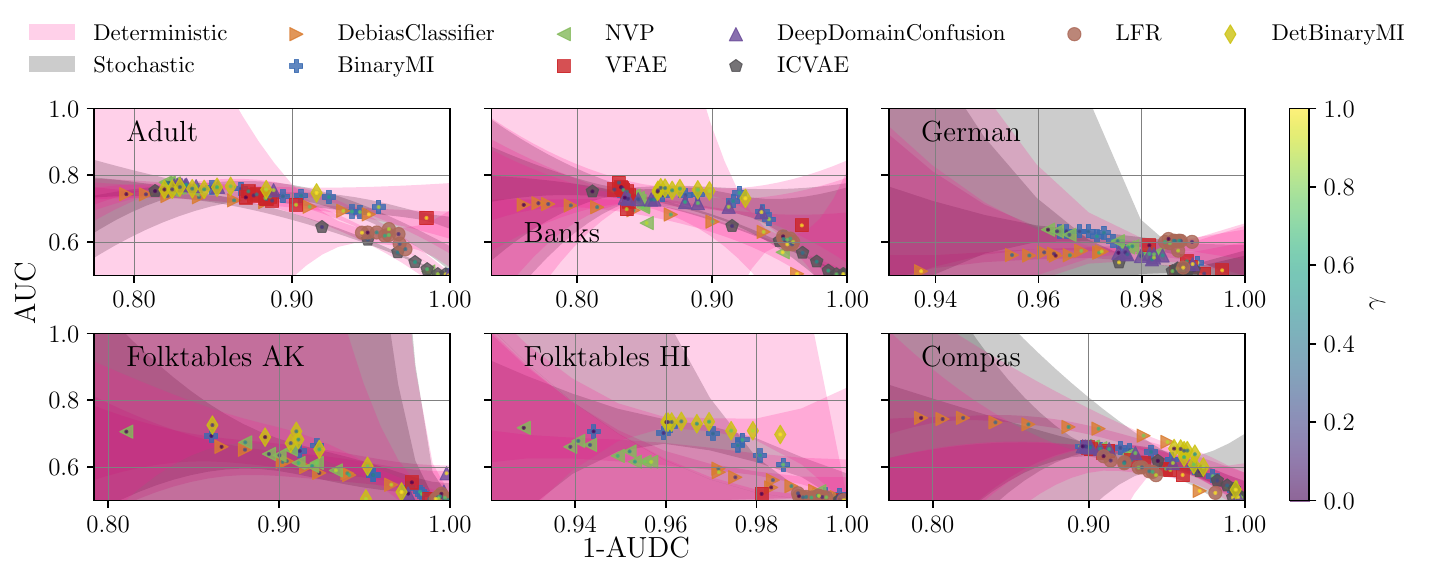}
    \caption{AUC vs. 1 - AUC-Discrimination tradeoff for all six dataset and eight model combinations. Each model is displayed for different $\gamma$ values indicated via a colored point inside the model marker.}
    \label{fig:model-auc-auc-dis}
\end{figure}

\subsection{Results in Invariant Representations}\label{sec:invariant}
\begin{figure}[h]
    \centering
    \includegraphics[width=\linewidth]{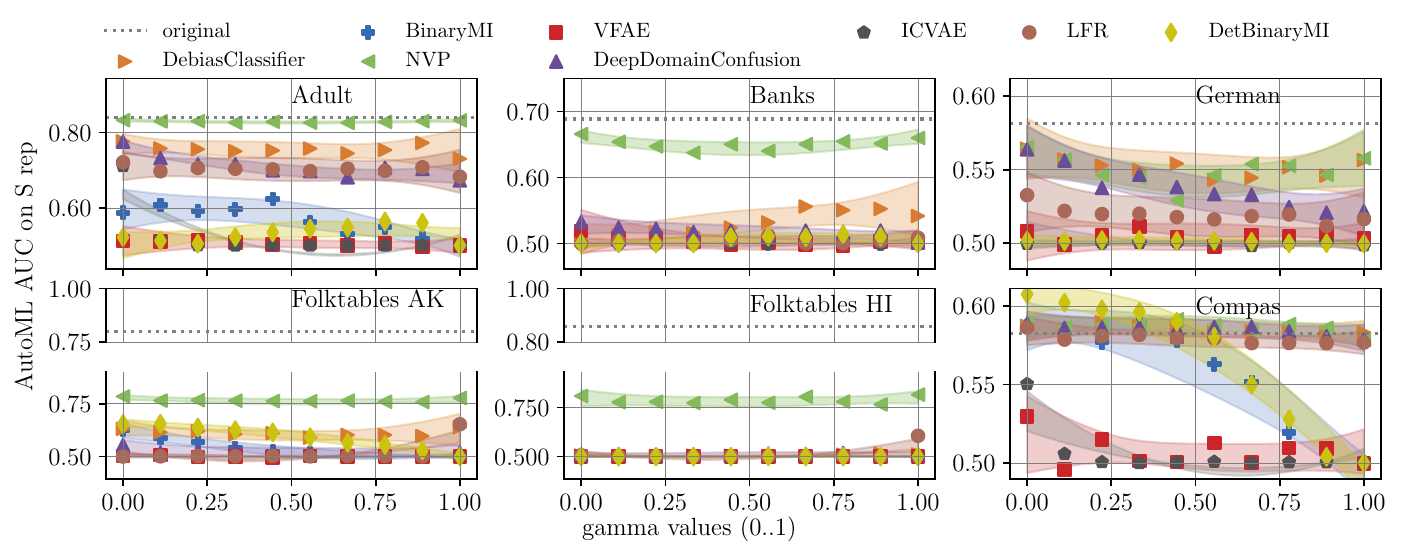}
    \caption{AutoML AUC vs. gamma results for all six dataset and eight model combinations.}
    \label{fig:automl-acc}
\end{figure}
We now take the same models reported in the previous subsection and investigate whether AutoML is able to recover information about the sensitive attributes from their representations.
Here, we expect that the accuracy of AutoML will approach the proportion of the majority group in the dataset (Figure \ref{fig:automl-acc}), and that the AUC will approach 0.5 (Figure~\ref{fig:automl-auc}), as the tradeoff parameter $\gamma$ increases. 
The results are particularly striking: while this is indeed what happens for stochastic or quantized models (BinaryMI, DetBinaryMI, VFAE, ICVAE) at higher $\gamma$ values, deterministic models have serious issues removing information (DebiasClassifier, NVP, DeepDomainConfusion, LFR) as the performance of AutoML remains well above random guess at many or all settings of $\gamma$.
These results experimentally confirm the theoretical impossibility theorem of Section \ref{sec:challenge}, and are of particular concern as they regard models that take overall fair allocation decisions. 
For instance, the DebiasClassifier is able to learn fair allocations on COMPAS; however, the representations still contain information about $S$ and should therefore not be considered safe for distribution to data users interested in employing them in other ML tasks.
We give a comparison of a ReLU-activated DebiasClassifier in the Appendix, Section~\ref{sec:app-relu}, where we observe similarly that information is not consistently removed.
A similar trend is clearly visible for the NVP model across all datasets.

We conclude by offering an explanation for the phenomenon of fair allocation models not learning invariant representations. 
We note that the observation that when $S$ and $Y$ are correlated, a weak classification model will also be relatively fair in terms of allocation.
Thus, FRL methodologies which are not severely tested for representation invariance may still obtain fair allocation decisions via a weak classification stage $Z^i \to \hat{Y}$.
\begin{figure}[h]
    \centering
    \includegraphics[width=\linewidth]{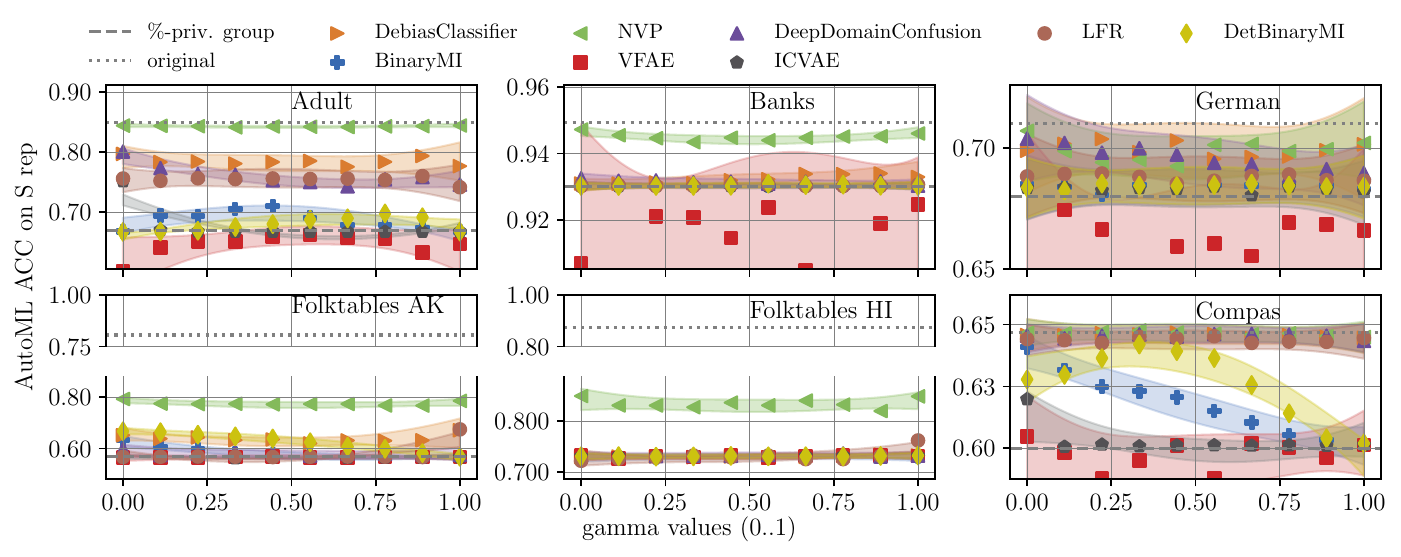}
    \caption{AutoML ACC vs. gamma results for all six dataset and eight model combinations.}
    \label{fig:automl-auc}
\end{figure}

\section{Fair Representation Learning: The Next 10 Years}
In this paper we have discussed a fundamental theoretical challenge to fair representation learning and experimentally analyzed its relevance to several methodologies proposed in the first ten years of this field.
To ensure that FRL develops into an influential methodology and achieves real-world impact, we put forward the following recommendations for further research. 
\begin{itemize}
    \item \textbf{Clarify information reduction strategy.} As Theorem \ref{th:theorem1} shows, many common assumptions in deep neural network learning (deterministic representations, injective activation functions) lead to serious theoretical FRL challenges. Future FRL research proposing new methodologies should discuss these fundamental information-theoretic results and clarify how the mutual information $I(T^i;S)$ may be actually reduced. Models that are currently understood to be information-reducing include stochastic \cite{Louizos2016TheVF} or highly quantized \cite{balunovic2021fair,cerrato2023invariant} representations.
    \item \textbf{Severe testing across both FRL frames.} As highlighted in Section \ref{sec:invariant}, it may happen that FRL methods will display fairness in terms of allocation but will not be in terms of learning invariant representations. Both evaluation frames are critical, especially since FRL methods are overall opaque and other simpler methods are provably optimal in term of fair allocation \cite{hardt2016equality}. To facilitate future severe testing in FRL we release \texttt{EvalFRL}, our extensible experimentation library \url{https://anonymous.4open.science/r/EvalFRL}.
    \item \textbf{Testing on datasets with known distributions.} One way to obtain rigorous baselines for information removal is to obtain and test on datasets for which the distributions are known a priori. To avoid testing on simplified toy datasets, recent developments in data generators for biased data \cite{biasondemand} should be considered. We elaborate on other sources for complex real-world data with known distributions and show initial results in Appendix \ref{sec:physics-data}.
\end{itemize}

\printbibliography

\newpage

\appendix

\section{Detailed Experimental Setup}\label{app:exp}

\begin{figure}[h]
    \centering
    \includegraphics[width=\linewidth]{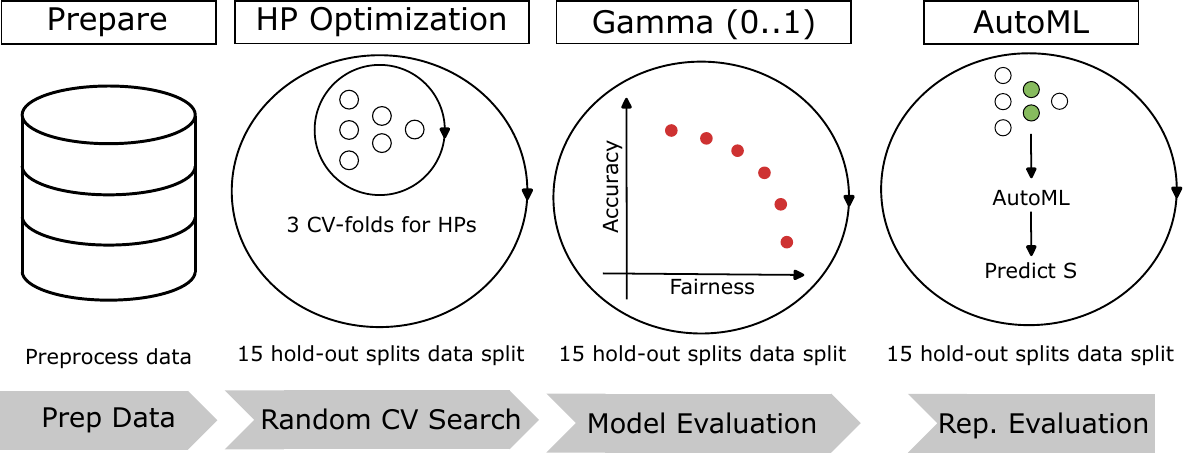}
    \caption{A graphical summary of \texttt{EvalFRL}, our experimentation library for FRL algorithms. It shows an overview of the Kedro pipeline used for preprocessing, hyperparameter optimization, $\gamma$ experiments and AutoML evaluation.}
    \label{fig:framework}
\end{figure}
We show a graphical summary of \texttt{EvalFRL} in Figure \ref{fig:framework}.
Our experimentation library employs separate Kedro pipelines to perform preprocessing, hyperparameter optimization, tradeoff analysis and the final evaluation for each model/dataset combination.
In detail, \texttt{EvalFRL} performs the following steps for every available dataset and model:
\begin{enumerate}
    \item Data preprocessing via encoding (discrete features) and normalization (continuous features).
    \item Employing a 15-by-3 hold-out/CV split \cite{nadeau2003inference}, tune the hyperparameters via random search so to maximise the AUC w.r.t. the classification task for each dataset. We test 100 different hyperparameter combinations.
    \item Utilizing the found hyperparameter combinations, models are then re-trained with gamma values ranging from 0 to 1. The trained models are used generate fair representations $Z$ by computing the activations of the second-to-last layer $Z^{L-1} = \phi^{L-1}(x)$.
    \item Use AutoML to predict the sensitive attribute $S$ from the representations $Z^{L-1}$.
\end{enumerate}
This procedure repeated across 8 models and 6 datasets implies a total of 335.000 model fits.
The tested hyperparameter ranges for each methodology and dataset are available at \url{https://anonymous.4open.science/r/EvalFRL/runs/experiment.yml}. The best hyperparameters found for each outer fold are available in the experimental metadata \url{https://drive.google.com/drive/folders/1koZd8cgBJMVGuH3uRqvpTFEUJo0Sd23q?usp=sharing}. We refer to the \texttt{README.md} file contained therein for instructions.

\subsection{Models}\label{sec:app-models}
\textbf{BinaryMI}
The BinaryMI model leverages stochastically activated binary layer(s) to compute the mutual information between these layers and the sensitive attributes.
By treating neurons as bernoulli random variables, this approach directly calculates the mutual information, which is then used as a regularization factor during gradient descent to ensure fairness in the learned representations \cite{cerrato2023invariant}.
n our experiments, we also utilize this model to determine whether the fairness of the representations is due to the Bernoulli layer or the quantization process.
Both factors, as discussed in Section \ref{sec:challenge}, may be employed to circumvent Theorem \ref{th:theorem1}.
To achieve this, we remove the Bernoulli sampling from the training phase and use a quantized sigmoid activation instead.
We refer to this deterministic version of the BinaryMI model as DetBinaryMI.

\textbf{DebiasClassifier}
The DebiasClassifier leverages adversarial training to create fair representations by integrating an adversarial network that discourages the encoding of protected attributes.
We implemented this method via gradient reversal as proposed by Ganin et al. \cite{ganin2016jmlr} in domain adaptation and then employed in fairness by Xie et al. \cite{xie2017controllable} and Madras et al. \cite{DBLP:journals/corr/abs-1802-06309}.

\textbf{NVP}
Out of several available FRL normalizing flow methodologies, we test a fair normalizing flow model \cite{corrvector} that leverages two RealNVP \cite{dinh2017realnvp} models.
We choose this method as it does not require the sensitive attribute at test time (differently to e.g. \cite{balunovic2021fair}) and as it does seek to break the bijective relationship inherent to normalizing flows (e.g. present in AlignFlow \cite{alignflow}) by ``funnelling'' information about sensitive attribute into one latent variable, which is then set to zero when entering the second RealNVP.

\textbf{VFAE}
The Fair Variational Autoencoder (VFAE) is a methodology proposed by Louizos et al. \cite{Louizos2016TheVF} that leverages variational autoencoders to learn fair representations. 
The fairness of the representations is obtained via architectural constraints and a loss term based on the Maximum Mean Discrepancy (MMD) \cite{gretton2012kernel}.

\textbf{ICVAE}
The Independent Conditional Variational Autoencoder (ICVAE) is also based on variational autoencoders \cite{moyer2018invariant}. 
Here, fairness is obtained via the well-known relationship between mutual information and the KL divergence.
A probability density for $P(Z)$ is made available by employing variational autoencoders. 

\textbf{LFR} Learning Fair Representations (LFR) was proposed by Zemel et al. \cite{zemel2013learning} and poineered the field of FRL. In LFR every individual gets stochastically mapped to so-called prototypes, which are points in the same space as $X$. This mapping $g: X \to Z$ combined with another mapping $f: Z \to Y$ are optimized to satisfy three goals: 1. $g$ statisfies group fairness, 2. $g$ retains all information on $X$ and 3. $f \circ g$ is close to the real classification.
While the formulation in the original paper \cite{zemel2013learning} appears to us to be discrete and stochastic, we note that its implementation in the \texttt{AIF360} library consistently returns continuous representations without any variance across different calls of the \texttt{transform(X)} function. 
We employed the \texttt{AIF360} implementation in our experimentation.

\textbf{DeepDomainConfusion}
\cite{tzeng2014deep} was introduced by Tzeng et al. as a domain adaptation method. Similarly to VFAE \cite{Louizos2016TheVF}, it employs the MMD \cite{gretton2012kernel} between representations of different domains as a loss function term. We instead encode domains as different values of sensitive attributes.

\subsection{Dataset Information}\label{sec:app-datasets}
\textbf{COMPAS.} This dataset (called Compas in the following), introduced by ProPublica \cite{machine_bias}, focuses on evaluating the risk of future crimes among individuals previously arrested, a system commonly used by US judges.
The ground truth is whether an individual commits a crime within the following two years.
The sensitive attribute is ethnicity.

\textbf{Adult.} This dataset, available in the UCI repository \cite{dua2019uci}, pertains to determining whether an individual's annual salary exceeds \$50,000.
We take gender to be the sensitive attribute \cite{Louizos2016TheVF,zemel2013learning}.

\textbf{Bank marketing.} Here, the classification task involves predicting whether an individual will subscribe to a term deposit.
This dataset (called Banks in the following) exhibits disparate impact and disparate mistreatment concerning age, particularly for individuals under 25 and over 65 years old.
\cite{banks}

\textbf{German.} The German Credit dataset, contains credit data of individuals with the objective of predicting their credit risk as either high or low risk. The gender of the individuals serves as the sensitive attribute \cite{misc_statlog_german_credit_data_144}.

\textbf{Folktables.} The folktables datasets are a collection of datasets derived from US cencus data, which span multiple years and all states of the USA \cite{ding2021retiring}.
Although the dataset supports multiple prediction tasks, we only used the income task, in which the objective is to predict whether an individual´s income exceeds 50.000\$. The sensitive attribute is the race of the individual. 
We picked the datasets from Alaska (AK) and Hawaii (HI) for our experiments, by comparing the performance of AutoML and logistic regression in predicting the sensitive attribute $S$ using the features $X$. 
We observed that on AK and HI AutoML performed remarkably better than linear models, indicating a complex relationship between the features $X$ and the sensitive attribute $S$.
Thus, we concluded that learning invariant representations on these datasets would be a relatively hard task.

\subsection{Other Fairness Metrics}
Before introducing further results in fair allocation, we report here other two classical fair allocation metrics commonly employed in the literature.

\paragraph*{Statistical Parity Difference}
Statistical Parity Difference (SPD) measures the difference in the probability of favorable outcomes between protected and unprotected groups.
It is defined as:

\begin{equation}
\text{SPD} = P(\hat{Y} = 1 \mid S = 0) - P(\hat{Y} = 1 \mid S = 1)
\end{equation}
where $\hat{Y}$ is the predicted outcome, and $S$ is the sensitive attribute (e.g., gender, race).
A value of 0 indicates perfect fairness, while values closer to -1 or 1 indicate higher disparity.

\paragraph*{Delta} Introduced by Zemel et al. \cite{zemel2013learning}, Delta is defined as $\text{Delta} = \text{yDiscrim} - \text{yAcc}$, with $\text{yAcc}$ being the prediction accuracy

\begin{equation}
    \text{yAcc} = 1 - \frac{1}{N}\sum_{n=1}^N |y_n - \hat{y}_n|.
\end{equation}

\cite{zemel2013learning}
This metric indicates the relative gain in terms of fairness vs. accuracy. 

\section{Other Results in Fair Allocation}

The following plots demonstrate how different models perform in terms of accuracy and fairness across a range of gamma values.
Deterministic models (pink shading) generally show higher accuracy and less variability compared to stochastic models (gray shading).
Notably, datasets like Banks and German exhibit more significant changes, highlighting the importance of gamma tuning.

\begin{figure}[!h]
    \centering
    \includegraphics[width=\linewidth]{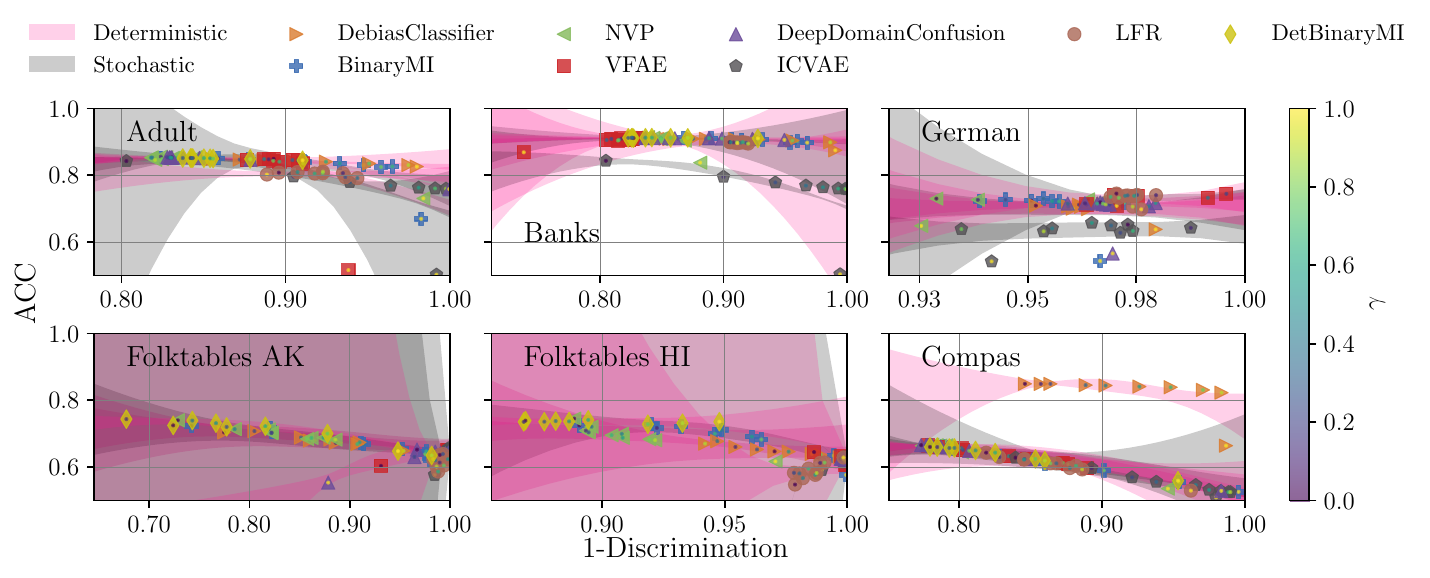}
    \caption{Accuracy vs. 1 - Discrimination tradeoff for all six dataset and eight model combinations.}
    \label{fig:model-acc-dis}
\end{figure}

\begin{figure}[!h]
    \centering
    \includegraphics[width=\linewidth]{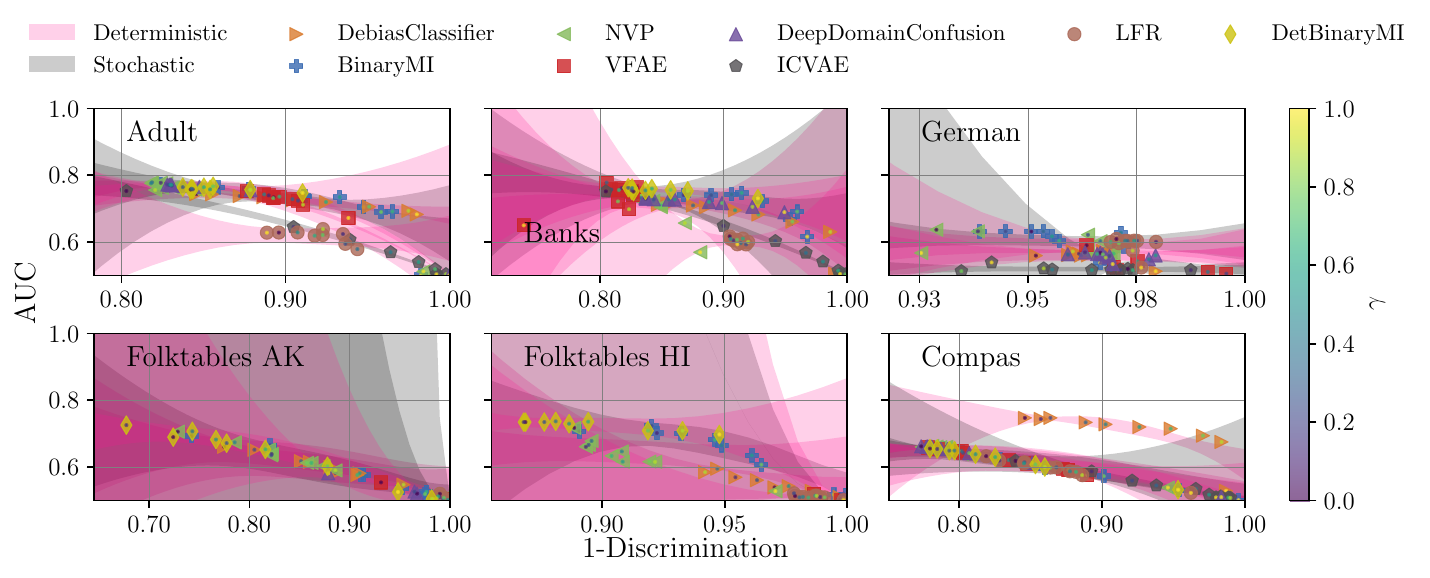}
    \caption{AUC vs. 1 - Discrimination tradeoff for all six dataset and eight model combinations.}
    \label{fig:model-auc-dis}
\end{figure}

\begin{figure}[!h]
    \centering
    \includegraphics[width=\linewidth]{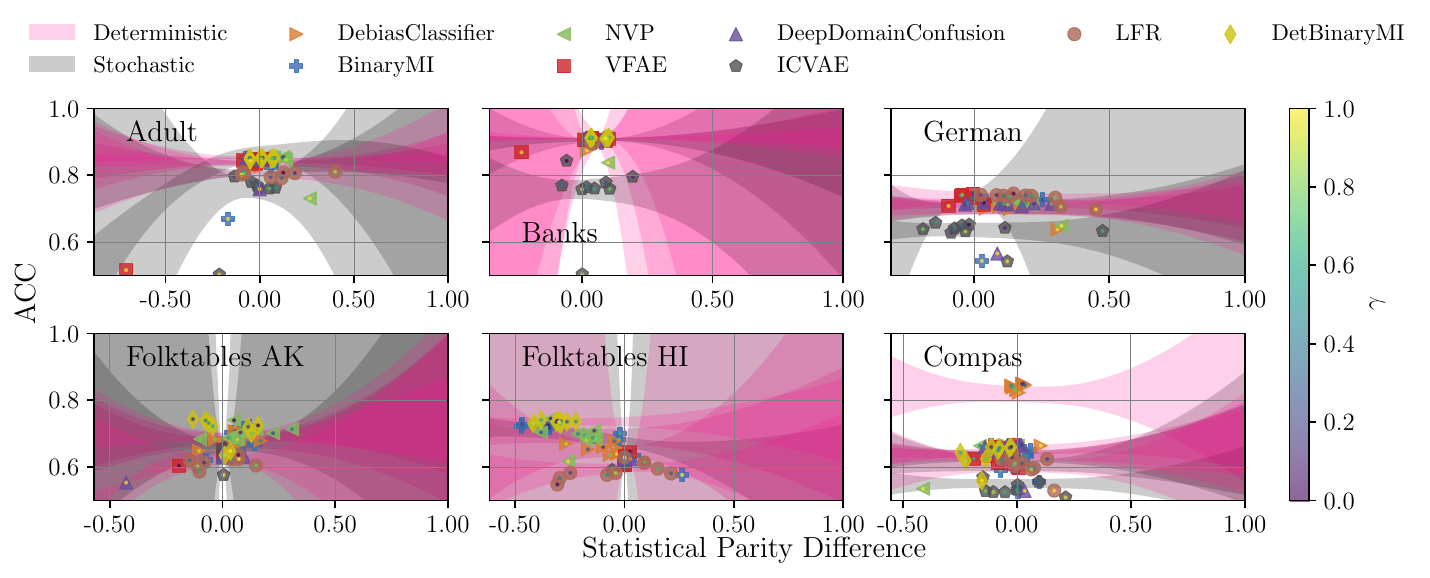}
    \caption{Accuracy vs. 1 - statistical parity difference tradeoff for all six dataset and eight model combinations.}
    \label{fig:model-acc-parity}
\end{figure}

\begin{figure}[!h]
    \centering
    \includegraphics[width=\linewidth]{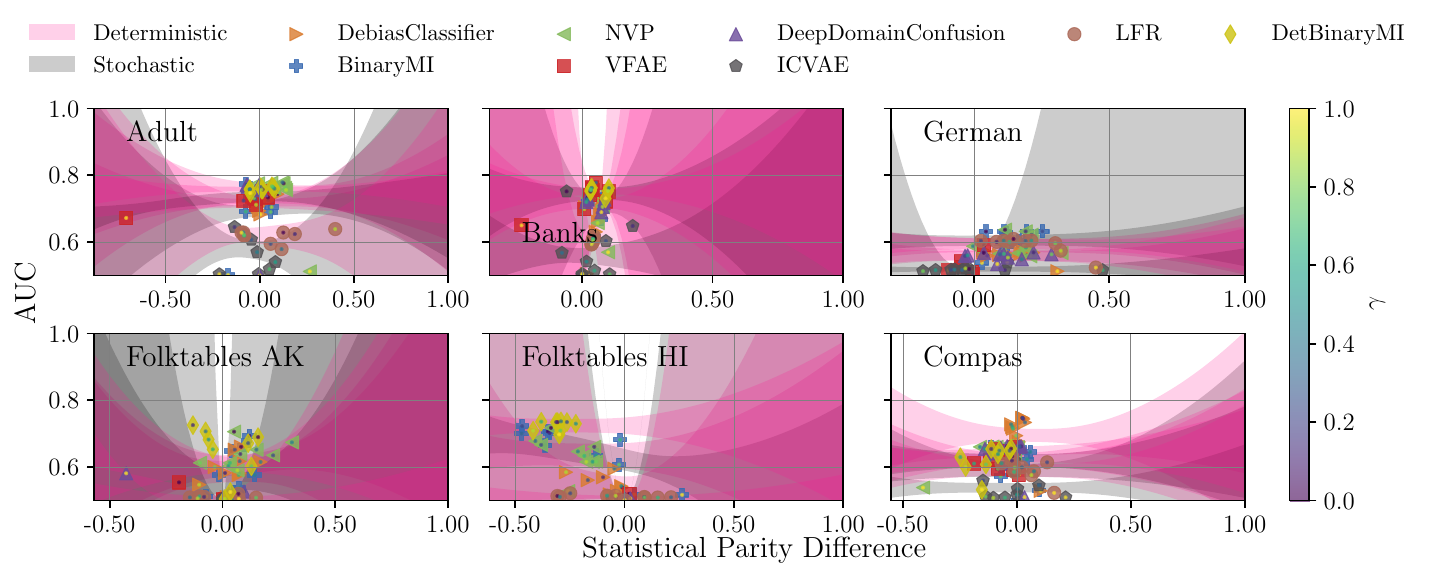}
    \caption{AUC vs. 1 - statistical parity difference tradeoff for all six dataset and eight model combinations.}
    \label{fig:model-auc-parity}
\end{figure}

\begin{figure}[!h]
    \centering
    \includegraphics[width=\linewidth]{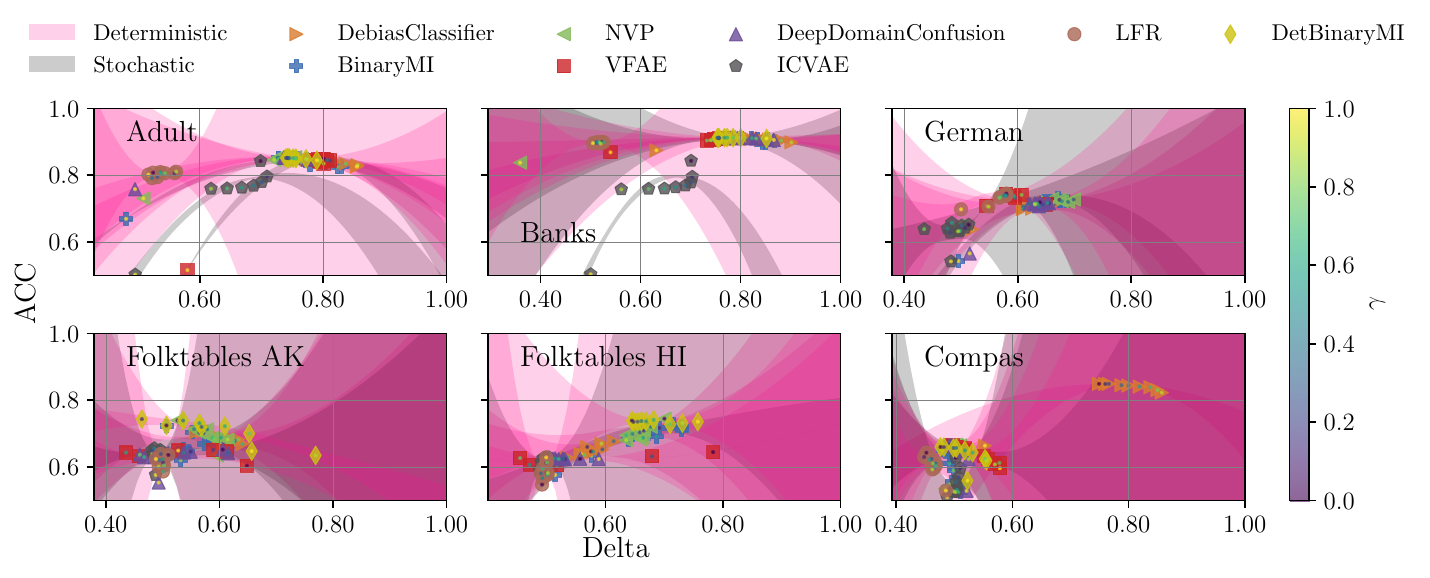}
    \caption{Accuracy vs. delta tradeoff for all six dataset and eight model combinations.}
    \label{fig:model-acc-delta}
\end{figure}

\begin{figure}[!h]
    \centering
    \includegraphics[width=\linewidth]{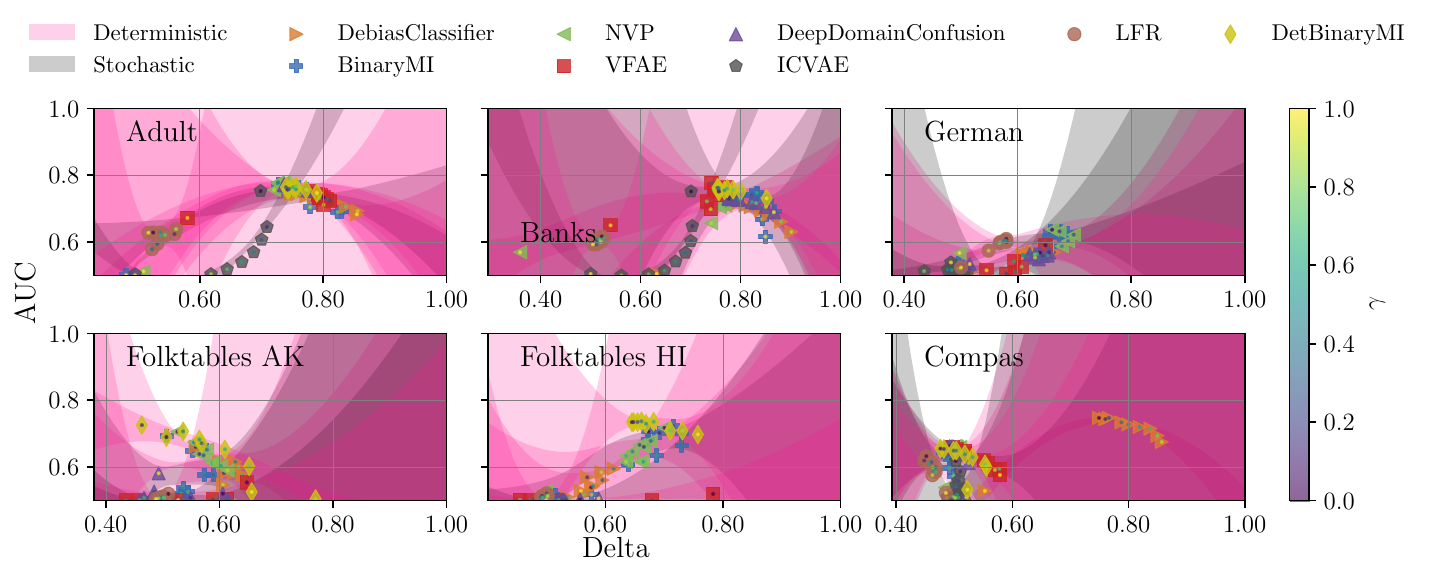}
    \caption{AUC vs. 1 - rND tradeoff for all six dataset and eight model combinations.}
    \label{fig:model-auc-delta}
\end{figure}

\begin{figure}[!h]
    \centering
    \includegraphics[width=\linewidth]{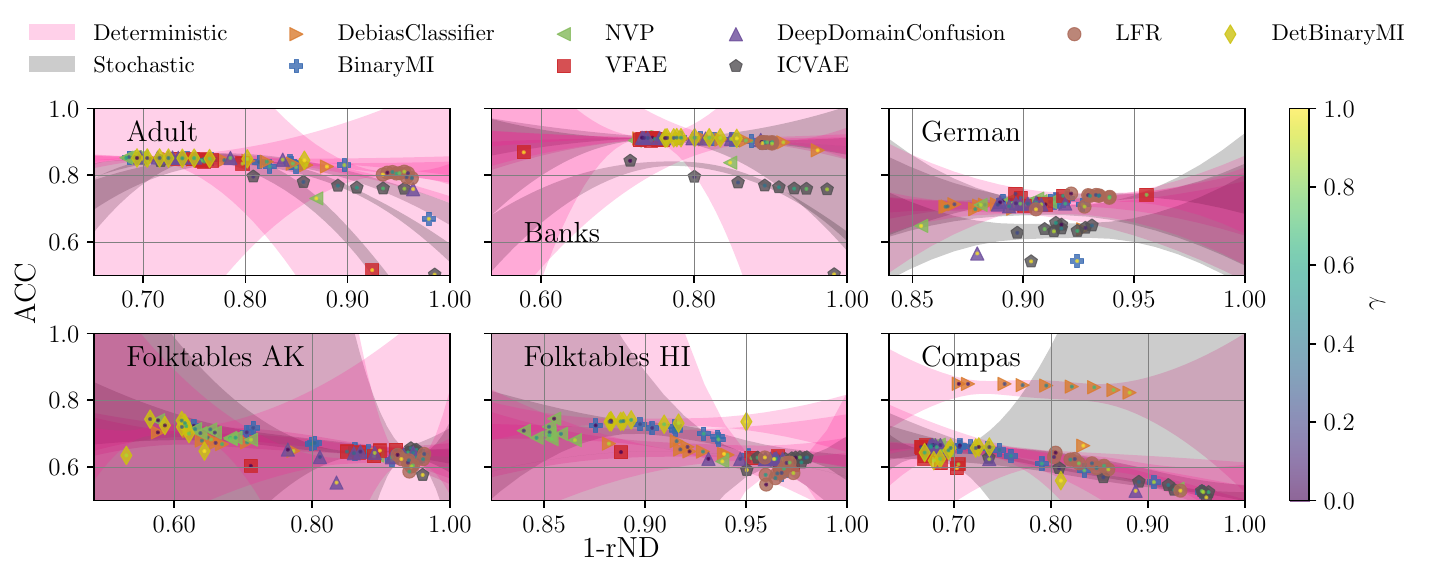}
    \caption{Accuracy vs. delta tradeoff for all six dataset and eight model combinations.}
    \label{fig:model-acc-rnd}
\end{figure}

\begin{figure}[!h]
    \centering
    \includegraphics[width=\linewidth]{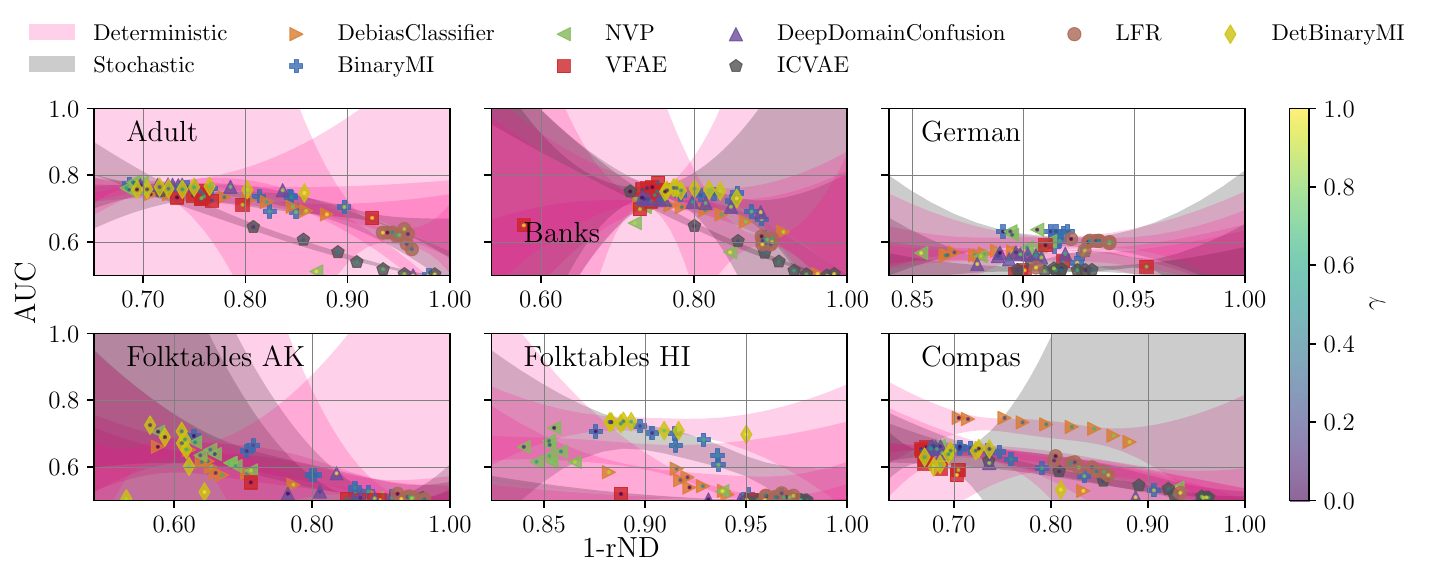}
    \caption{AUC vs. 1 - rND tradeoff for all six dataset and eight model combinations.}
    \label{fig:model-auc-rnd}
\end{figure}

\FloatBarrier
\section{ReLU Activation Tests}\label{sec:app-relu}

\subsection{ReLU Model Tests}
Figures \ref{fig:model-acc-dis-relu}, \ref{fig:model-auc-dis-relu}, \ref{fig:model-acc-parity-relu}, \ref{fig:model-auc-parity-relu},\ref{fig:model-acc-delta-relu}, \ref{fig:model-auc-delta-relu}, \ref{fig:model-acc-rnd-relu} and \ref{fig:model-auc-rnd-relu} show how DebiasClassifier and its ReLU variant perform across different gamma values in terms of accuracy and fairness.
Generally, the accuracy remains stable with slight variations across datasets, indicating that ReLU activation does not drastically alter the fairness-accuracy trade-off.
For the Compas the normal DebiasClassifier outperforms the ReLU variant significantly.

\begin{figure}[!h]
    \centering
    \includegraphics[width=\linewidth]{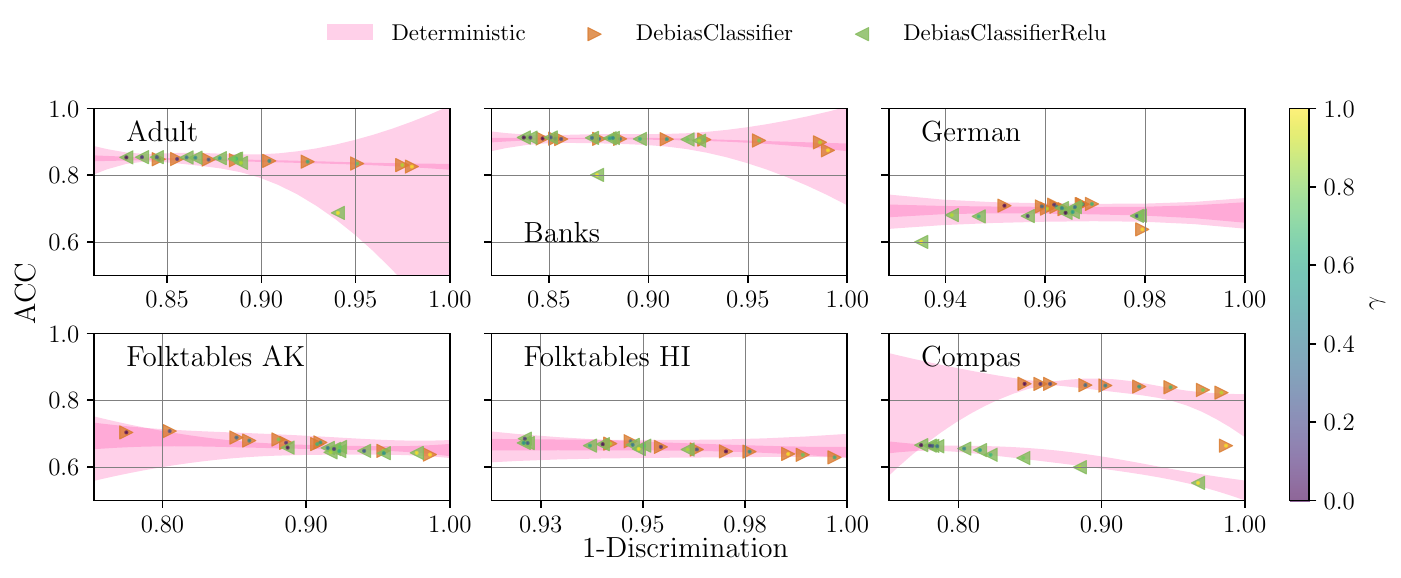}
    \caption{Accuracy vs. 1 - Discrimination tradeoff for the DebiasClassifier with $\tanh$ and ReLU activation for the six datasets.}
    \label{fig:model-acc-dis-relu}
\end{figure}

\begin{figure}[!h]
    \centering
    \includegraphics[width=\linewidth]{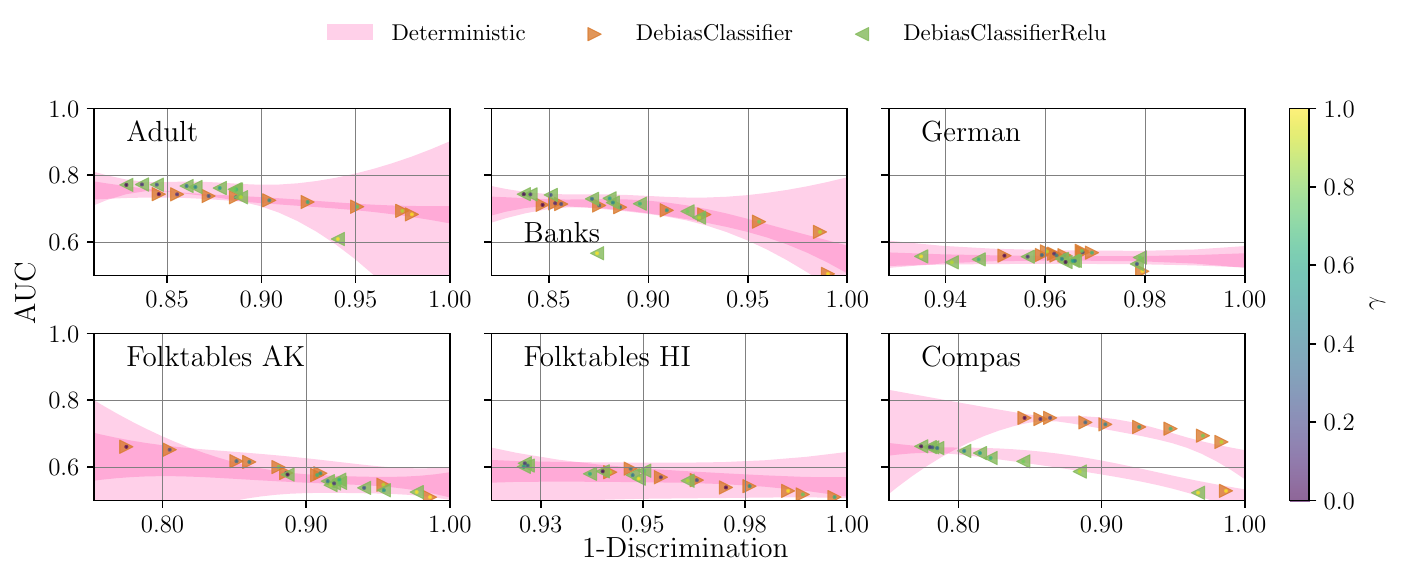}
    \caption{AUC vs. 1 - Discrimination tradeoff for the DebiasClassifier with $\tanh$ and ReLU activation for the six datasets.}
    \label{fig:model-auc-dis-relu}
\end{figure}

\begin{figure}[!h]
    \centering
    \includegraphics[width=\linewidth]{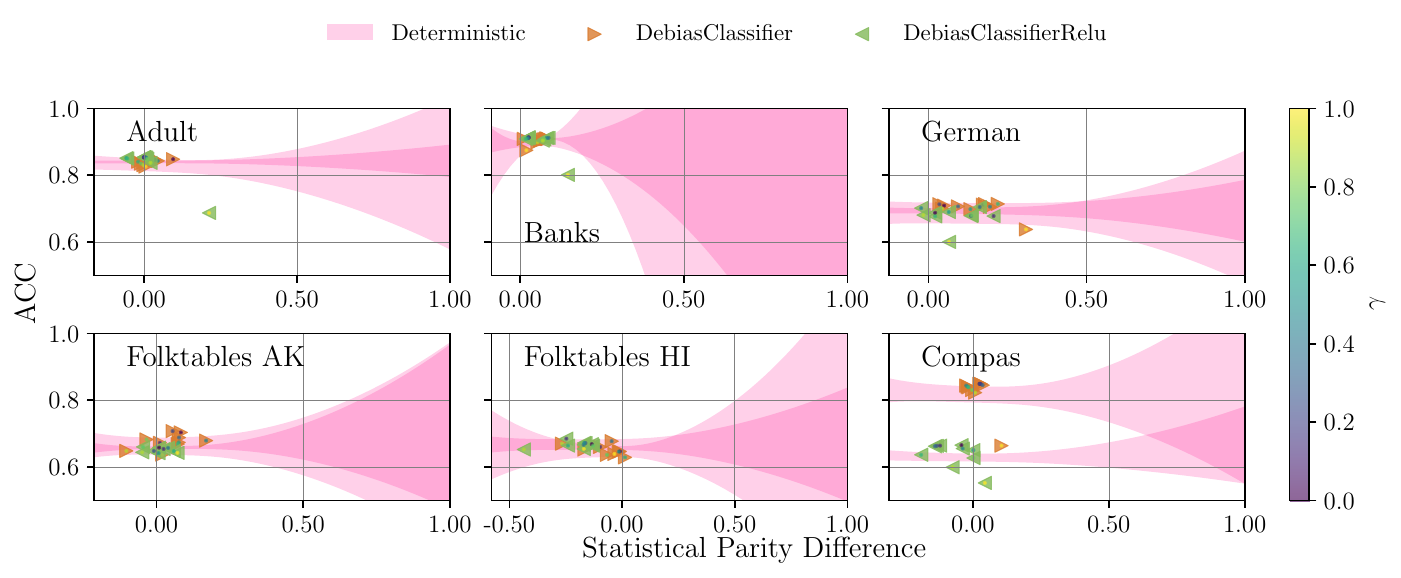}
    \caption{Accuracy vs. 1 - statistical parity difference tradeoff for the DebiasClassifier with $\tanh$ and ReLU activation for the six datasets.}
    \label{fig:model-acc-parity-relu}
\end{figure}

\begin{figure}[!h]
    \centering
    \includegraphics[width=\linewidth]{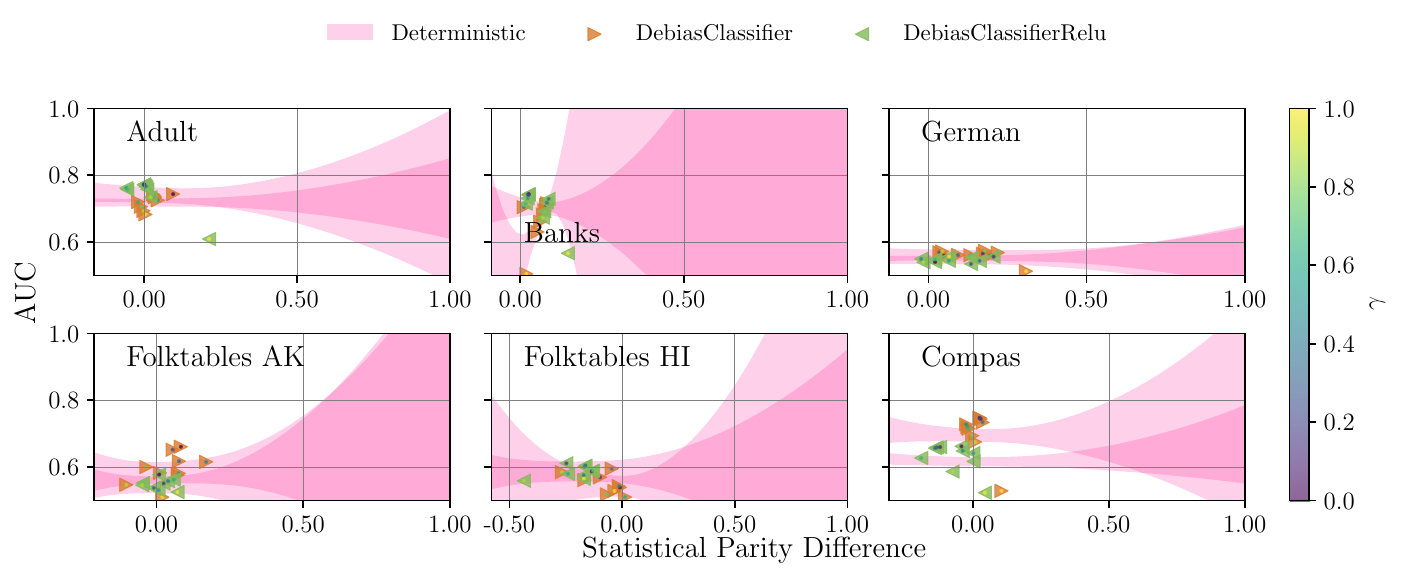}
    \caption{AUC vs. 1 - statistical parity difference tradeoff for the DebiasClassifier with $\tanh$ and ReLU activation for the six datasets.}
    \label{fig:model-auc-parity-relu}
\end{figure}

\begin{figure}[!h]
    \centering
    \includegraphics[width=\linewidth]{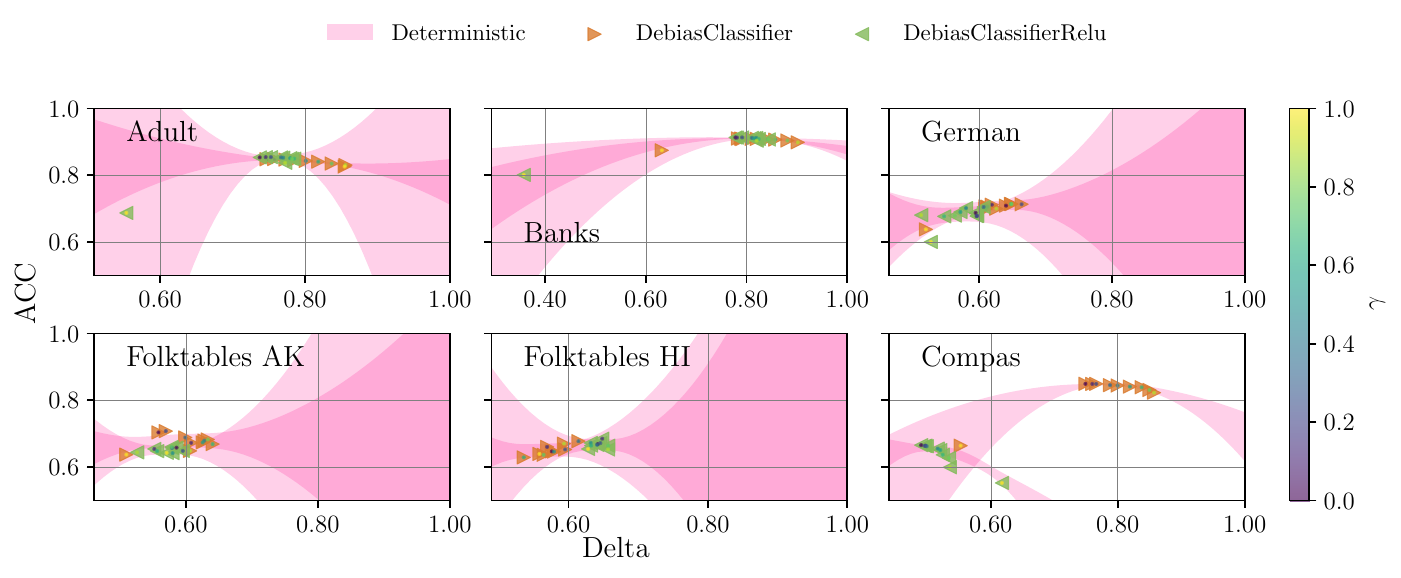}
    \caption{Accuracy vs. delta tradeoff for the DebiasClassifier with $\tanh$ and ReLU activation for the six datasets.}
    \label{fig:model-acc-delta-relu}
\end{figure}

\begin{figure}[!h]
    \centering
    \includegraphics[width=\linewidth]{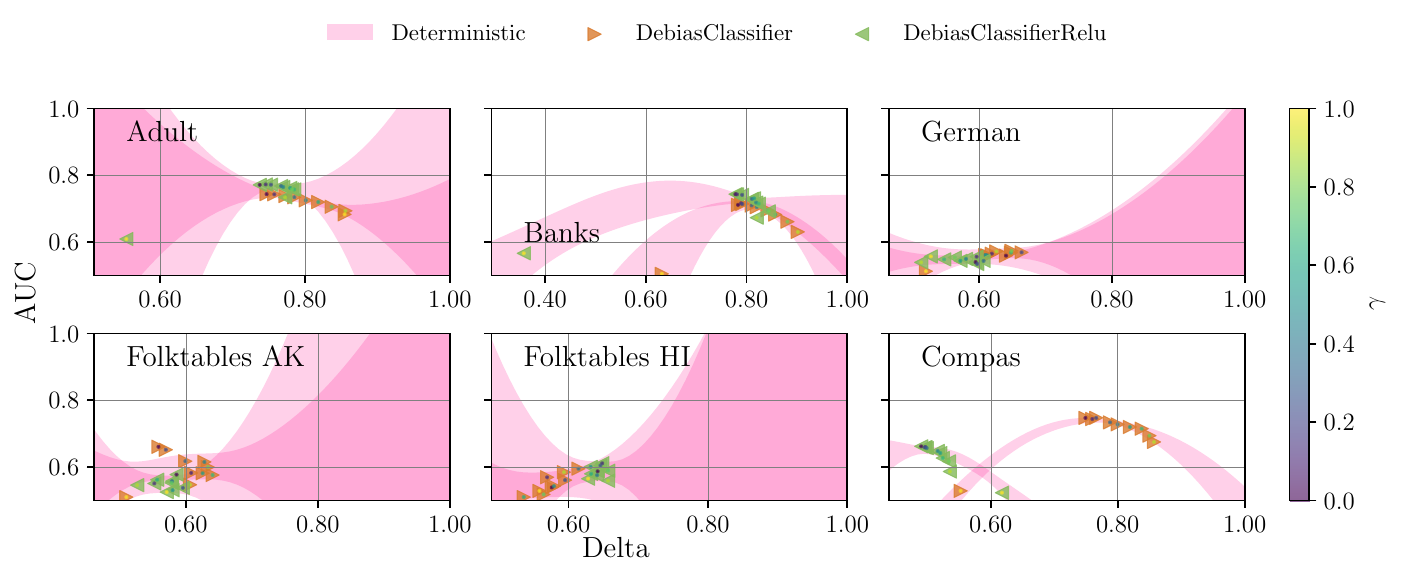}
    \caption{AUC vs. 1 - rND tradeoff for the DebiasClassifier with $\tanh$ and ReLU activation for the six datasets.}
    \label{fig:model-auc-delta-relu}
\end{figure}

\begin{figure}[!h]
    \centering
    \includegraphics[width=\linewidth]{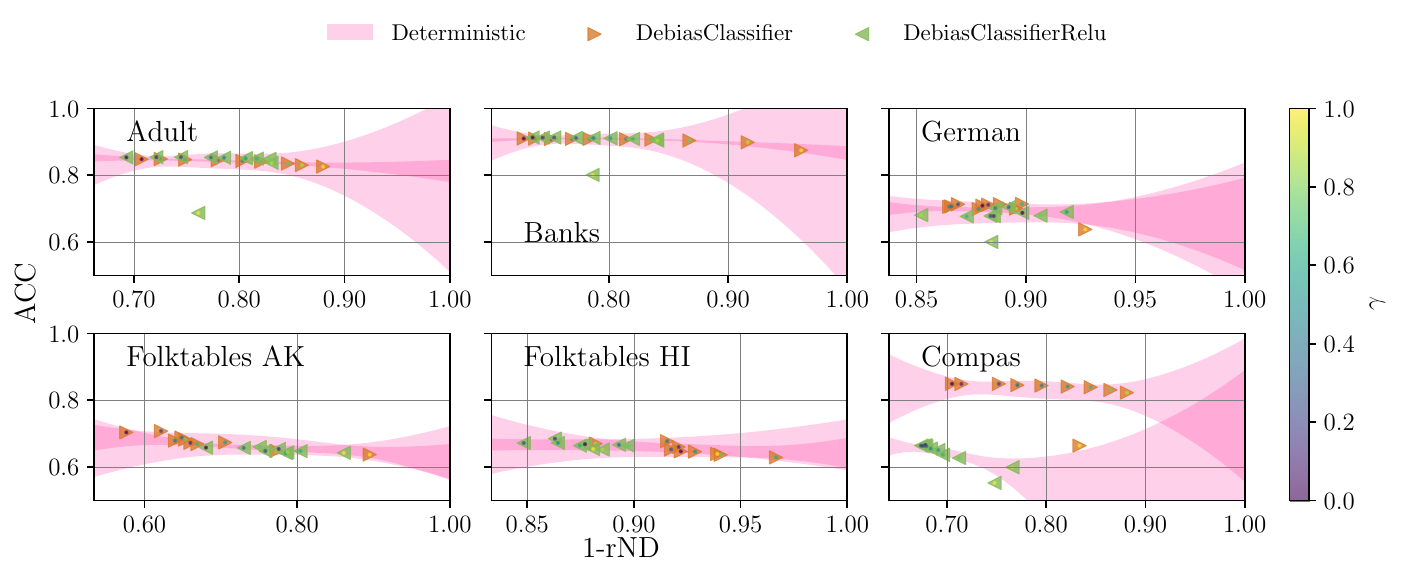}
    \caption{Accuracy vs. delta tradeoff for the DebiasClassifier with $\tanh$ and ReLU activation for the six datasets.}
    \label{fig:model-acc-rnd-relu}
\end{figure}

\begin{figure}[!h]
    \centering
    \includegraphics[width=\linewidth]{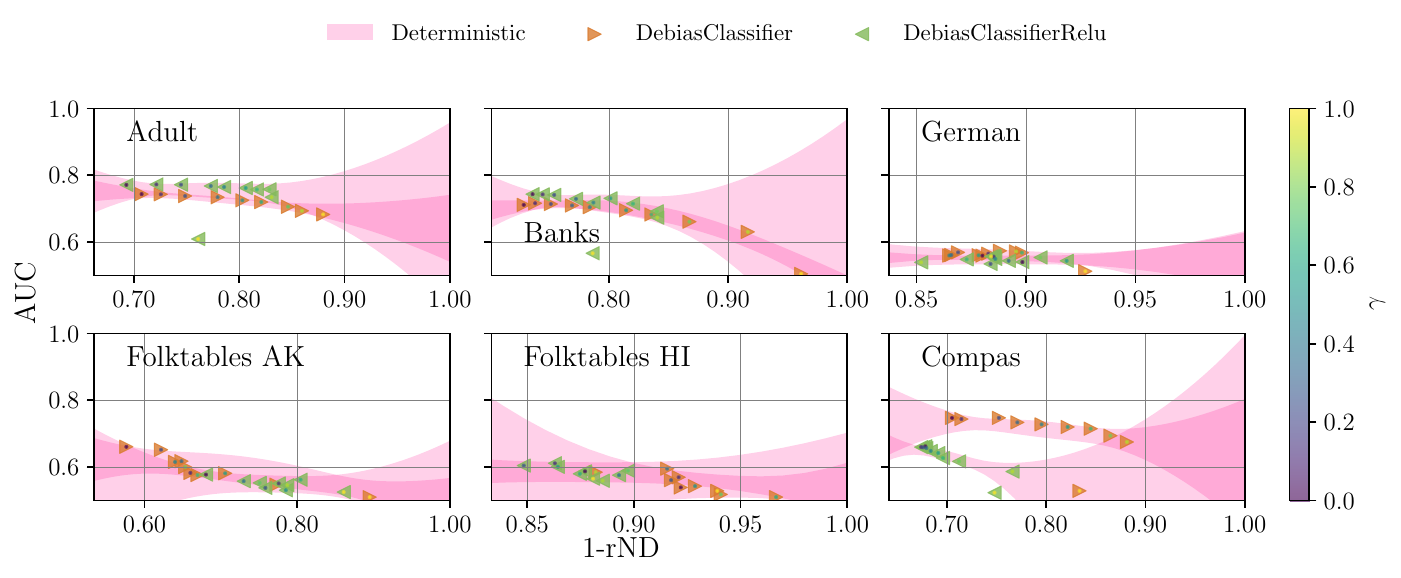}
    \caption{AUC vs. 1 - rND tradeoff for the DebiasClassifier with $\tanh$ and ReLU activation for the six datasets.}
    \label{fig:model-auc-rnd-relu}
\end{figure}

\FloatBarrier
\subsection{ReLU AutoML Tests}

Figures \ref{fig:automl-acc-relu} and \ref{fig:automl-auc-relu} illustrates the performance of AutoML to predict the sensitive attribute from the representations from the DebiasClassifier and DebiasClassifierRelu across a range of $\gamma$ values.
The results indicate that the AutoML accuracy and AUC for both representations maintain relatively stable results over different gamma values, and seems to be relatively constant across all considered $\gamma$ values.

\begin{figure}[h]
    \centering
    \includegraphics[width=\linewidth]{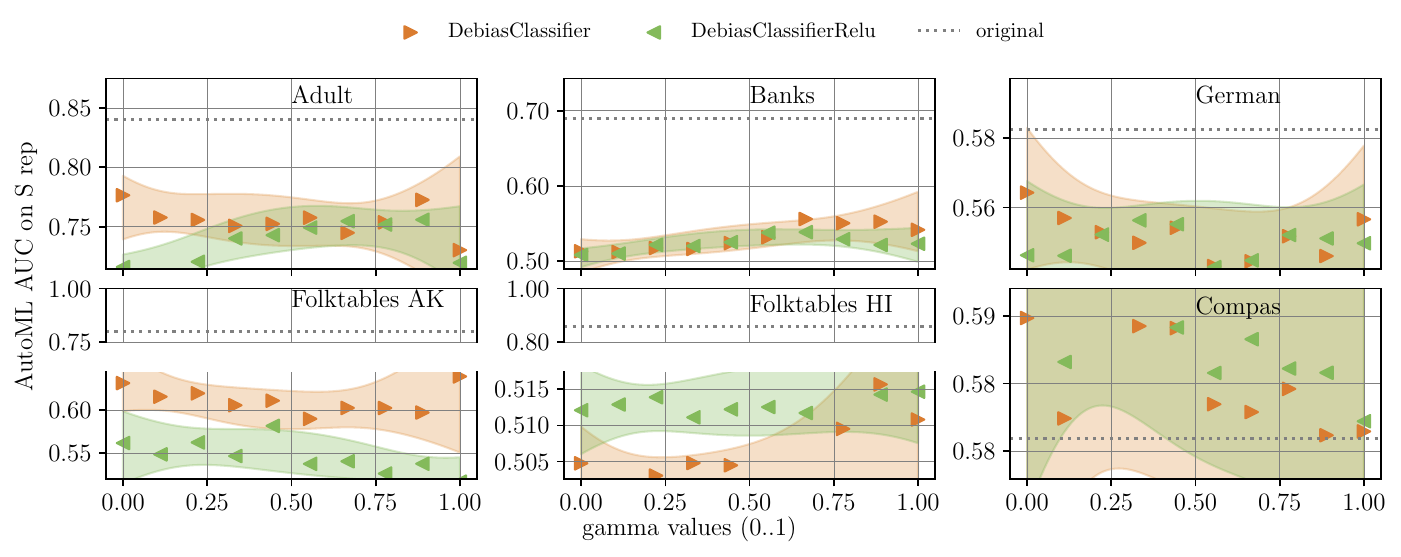}
    \caption{AutoML AUC of predicting $S$ versus gamma values on on the representations from DebiasClassifier and DebiasClassifierRelu across multiple datasets. The dotted lines represent the performance on the original data.}
    \label{fig:automl-auc-relu}
\end{figure}

\begin{figure}[h]
    \centering
    \includegraphics[width=\linewidth]{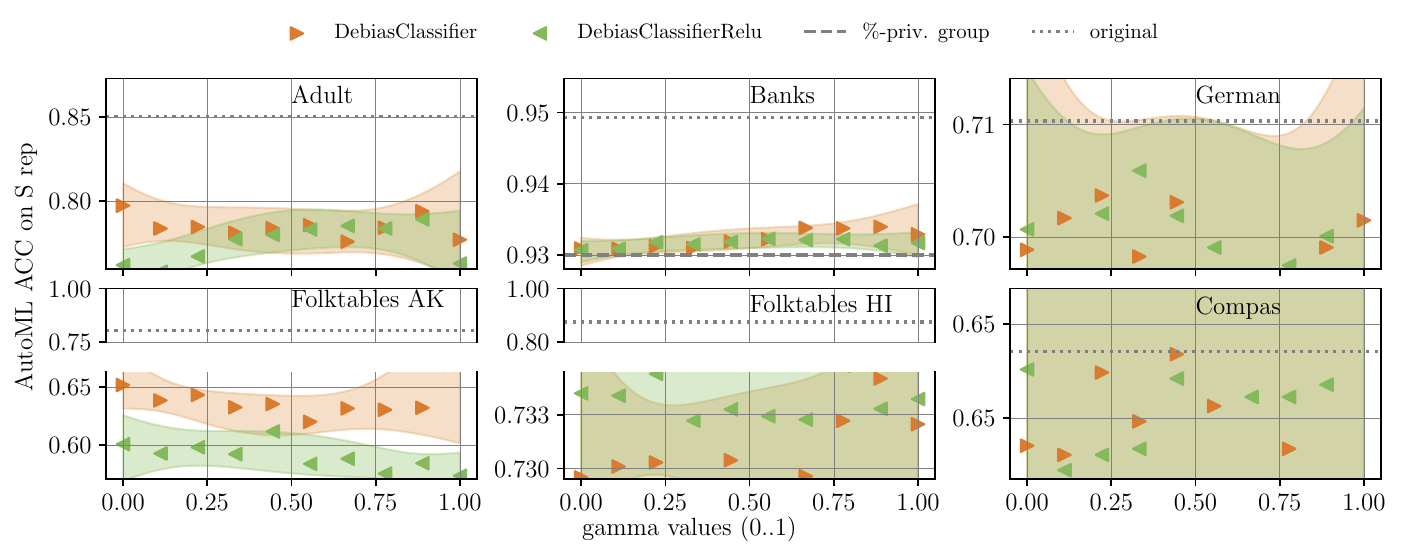}
    \caption{AutoML AUC of predicting $S$ versus gamma values on on the representations from DebiasClassifier and DebiasClassifierRelu across multiple datasets. The dashed and dotted lines represent the percentage of the privileged group and the performance on the original data, respectively.}
    \label{fig:automl-acc-relu}
\end{figure}

\FloatBarrier
\section{Particle Physics Data}\label{sec:physics-data}

To rigorously assess whether a machine learning model has effectively removed unwanted information, it is crucial to comprehend the underlying data generation process.
However, in many real-world scenarios where fairness is a concern, data is often sourced from human interactions, making it challenging to fully understand the origins of bias.

\begin{figure}[h]
    \centering
    \includegraphics[width=\linewidth]{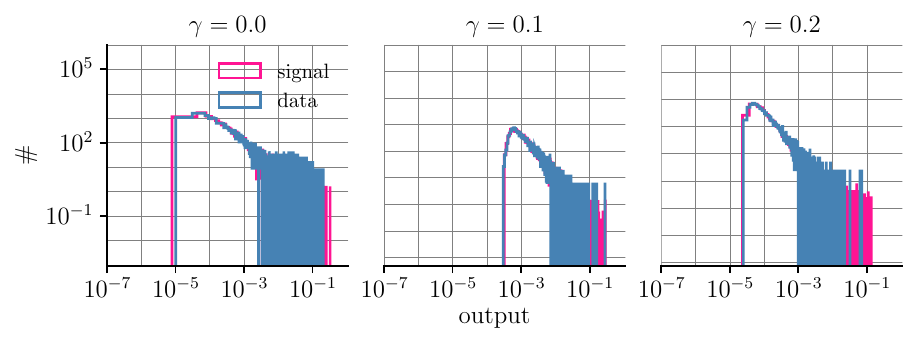}
    \caption{Comparison of signal and background data at various gamma ($\gamma$) levels using the BinaryMI model on the Kaggle 'Flavours of Physics' dataset. The plot showcases how different gamma values ($\gamma=0.0, 0.1, 0.2$) impact the distribution of the output variable, demonstrating the model's ability to be invariant between signal and background events.}
    \label{fig:binaryMI-output}
\end{figure}

\begin{figure}[h]
    \centering
    \includegraphics[width=\linewidth]{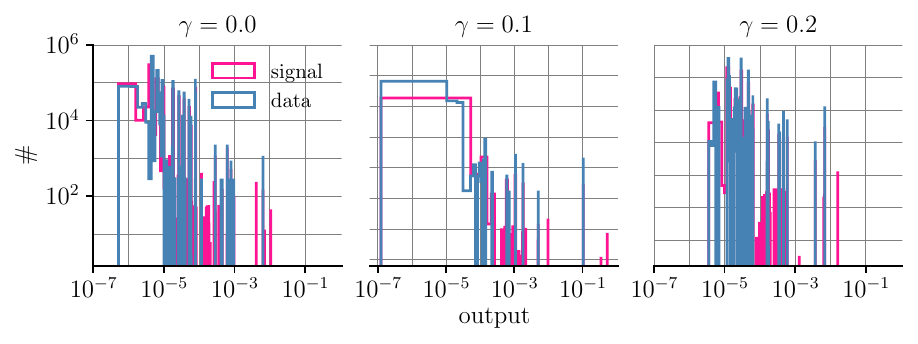}
    \caption{Comparison of signal and background data at various gamma ($\gamma$) levels using the DebiasClassifier model on the Kaggle 'Flavours of Physics' dataset. The plot showcases how different gamma values ($\gamma=0.0, 0.1, 0.2$) impact the distribution of the output variable.}
    \label{fig:debias-output}
\end{figure}
One approach to address this challenge while maintaining the complexity of real-world datasets is to explore domains where data creation and collection processes are meticulously documented and understood with a high degree of precision.
One such domain is particle physics.
Previous work tested adversarial FRL \cite{ganin2016jmlr} to predict W-jets events, being at the same time invariant to pileup (i.e. noise) \cite{10.5555/3294771.3294865}.
We note however that the original data was not published, to the best of our knowledge.
We therefore suggest to use the data from the ``Flavour of Physics'' Kaggle challenge \cite{kaggle-flavour}.
The task is to identify $\tau \rightarrow \mu \mu \mu$ decay events in high-energy physics data \cite{Aaij_2015} and improve the detection of this rare particle decay process using machine learning techniques.
Participants are provided with datasets containing particle collision data and are tasked with building models to distinguish between signal and background events.
The FRL/invariance framing is provided by an agreement test which quantifies the level of invariance obtained by the model across predictions for simulation and real (i.e. detector) data.
This test is necessary as for an unknown signal event there only simulation data will be available; while for the background event both simulation and real detector data are given.

To mitigate this issue most particle physics experiments have so-called ``control" areas where the real- and simulation-data distributions are well-understood theoretically.
These control areas are therefore employed in an agreement test.

When testing different FRL methodologies on this complex physical data, we observe phenomena which are consistent with our findings in Sections 3 and 4. 
In Figure \ref{fig:binaryMI-output} the output of the BinaryMI model -- a stochastic model -- over the control area is shown for the decay $Ds \rightarrow \varphi \Pi$ which has the same signature as the $\tau \rightarrow \mu \mu \mu$ decay \footnote{The code for this initial evaluation can be found at \url{https://anonymous.4open.science/r/EvalFRL/notebooks/distribution_shift.ipynb}.}.
When increasing the $\gamma$ the two output distributions become increasingly similar, indicating that the model is invariant to data and signal.

In contrast to the BinaryMI model, Figure \ref{fig:debias-output} shows the output of the DebiasClassifier.
Besides the discrete output values the DebiasClassifier shows greater separation between signal and background data at higher $\gamma$ values, which implies less fairness since more separation indicates a stronger bias towards certain data.

\FloatBarrier
\section{Computing Infrastructure}\label{sec:compute}
All experiments were run on CPUs without involving GPUs.
The system used consisted of four PCs, each equipped with 190 GB of RAM and an AMD EPYC 9254 24-Core Processor.
The total runtime for all experiments was approximately one week.
The primary limitation was not computational power but the required RAM to fit all models simultaneously.
For reproducing the experiments, we recommend running one model over all six datasets per PC to manage memory constraints effectively.

\end{document}